%% file: main.tex
\PassOptionsToPackage{x11names}{xcolor}
\documentclass{article}
\usepackage{iclr2027_conference,times}
\iclrfinalcopy

\input{math_commands}

\usepackage{xcolor}
\definecolor{DarkPurple}{RGB}{68,68,89}
\definecolor{LightPurple}{RGB}{126,126,166}
\definecolor{Olive}{RGB}{109,140,100}
\definecolor{DarkerPurple}{RGB}{41,41,89}
\definecolor{keylavender}{RGB}{245,245,250}
\definecolor{LinkBlue}{HTML}{2970CC}
\usepackage[colorlinks=true,linkcolor=LinkBlue,urlcolor=LinkBlue,citecolor=LinkBlue,breaklinks,pagebackref=true]{hyperref}
\usepackage{url}
\usepackage{microtype}
\usepackage{amsmath}
\usepackage{amsthm}
\usepackage{mathtools}
\usepackage{enumitem}
\usepackage{booktabs}
\usepackage{multirow}
\usepackage{nicematrix}
\usepackage{algorithm}
\usepackage{algpseudocode}
\usepackage{xspace}
\usepackage{thmtools}
\usepackage{thm-restate}
\usepackage[most]{tcolorbox}
\usepackage{wrapfig}
\usepackage{cleveref}
\usepackage{backref}
\crefname{equation}{}{}
\crefname{appendix}{Appendix}{Appendices}
\Crefname{appendix}{Appendix}{Appendices}

\renewcommand*{\backref}[1]{}
\renewcommand*{\backrefalt}[4]{\ifcase#1\relax\or(page #2)\else(pages #2)\fi}

\tcbset{propbox/.style={
  enhanced,
  colback=black!5, frame hidden, arc=6pt,
  boxsep=0pt, left=10pt, right=10pt, top=8pt, bottom=8pt,
  before skip=10pt, after skip=10pt,
  before upper={\parskip=0.5em\relax}
}}
\declaretheoremstyle[
  spaceabove=0pt, spacebelow=0pt,
  headfont=\bfseries, notefont=\mdseries, notebraces={(}{)},
  bodyfont=\itshape, headpunct={.}, postheadspace=0.5em
]{propstyle}
\declaretheorem[name=Proposition,style=propstyle,
  preheadhook={\tcolorbox[propbox]},
  postfoothook={\endtcolorbox}]{proposition}
\declaretheorem[name=Corollary,style=propstyle,sibling=proposition,
  preheadhook={\tcolorbox[propbox]},
  postfoothook={\endtcolorbox}]{corollary}

\newcommand{\ours}{\textit{a}FMLM}

\title{Asynchronous Flow Map Language Models}

\author{Chanhyuk Lee$^1$\quad Wonjung Kim$^1$\quad Patrick Pynadath$^2$\quad Jaehoon Yoo$^1$\\
\textbf{Satoshi Hayakawa$^3$\quad Ruqi Zhang$^2$\quad Jinwoo Kim$^{1\,\dag}$\quad Seunghoon Hong$^{1\,\dag}$}\vspace{0.3em}\\
$^1$KAIST\quad $^2$Purdue University\quad $^3$The University of Tokyo\\
$^{\dag}$Equal advising
}

\begin{document}

\maketitle
\lhead{Preprint.}

\begin{abstract}
\input{abstract}
\end{abstract}

\input{sections/intro}
\input{sections/related_work}

\input{sections/background}
\input{sections/method}

\input{sections/experiments}
\input{sections/conclusion}

\bibliography{references}
\bibliographystyle{iclr2027_conference}

\newpage
\appendix
\crefalias{section}{appendix}
\crefalias{subsection}{appendix}
\crefalias{paragraph}{appendix}
\input{sections/appendix_proofs}
\input{sections/appendix_schedulers}
\input{sections/appendix_experiments}

\end{document}

%% file: math_commands.tex
\usepackage{amsmath,amsfonts,bm}

\def\1{\bm{1}}

\DeclareMathAlphabet{\mathsfit}{\encodingdefault}{\sfdefault}{m}{sl}
\SetMathAlphabet{\mathsfit}{bold}{\encodingdefault}{\sfdefault}{bx}{n}



%% file: abstract.tex
The order in which tokens are predicted in language modeling has a profound impact on the quality of generated sequences.
Autoregressive and masked diffusion models iteratively fill in tokens and can exploit such orderings across generation steps, but how to leverage an ordering for one-step parallel language modeling is unclear.
We introduce \textbf{\emph{a}synchronous flow map language models (\ours{}s)}, which incorporate prediction order into one-step modeling by learning the flow maps of asynchronous flows.
These flows specify an ordering through a tokenwise time schedule that controls the relative denoising progress of tokens.
We formulate a general class of asynchronous flows that admit both prescribed and learned, state-adaptive schedules, show that they define valid flow maps, and develop new methods for their flow map distillation.
The resulting models can generate all tokens in a single forward pass by internalizing the ordering encoded in the flow into their hidden computations.
On Sudoku solving and code generation for GSM8K math word problems, \ours{} establishes new state-of-the-art one-step performance\footnote{Code is available at \url{https://github.com/david3684/async-fmlm}.}.

%% file: sections/intro.tex
\section{Introduction}

How tokens are ordered and scheduled for prediction has a significant impact on language modeling \citep{shannon1951prediction,papadopoulos2024arrows,kim2025train,nie2026large}, as it determines how sequence prediction is decomposed into a series of simpler modeling subproblems.
For instance, autoregressive models (ARMs) \citep{bengio2003neural} encode an ordering as a prescribed left-to-right factorization, whereas masked diffusion models (MDMs) \citep{sahoo2024simple,shi2024simplified} learn over diverse masking patterns and choose orderings for generation based on model outputs.
These choices would not matter for perfect models, as any factorization recovers the true joint distribution.
In practice, however, models have finite capacity and some subproblems are harder to learn than others, making generation quality sensitive to the chosen order and schedule \citep{finzi2026entropy}.

Because these orderings and schedules are defined over generation steps, in which tokens decided earlier inform later ones, iterative models such as ARMs and MDMs can naturally use them.
In contrast, parallel language models designed for fast generation take only a few steps and must perform more of their computation in hidden space, where it is unclear how to leverage a given ordering or schedule.
This problem is most severe in one-step generation~\citep{lee2026flowmaplanguagemodels}, which predicts all tokens in a single forward pass and thus leaves no steps over which an order can be followed.

We address this problem with \textbf{\emph{a}synchronous flow map language models (\ours{}s)} that incorporate prediction order into one-step generation.
Our idea builds on flow map language models (FMLMs) \citep{lee2026flowmaplanguagemodels}, which achieve one-step generation by compressing the sampling trajectories of continuous flow language models (FLMs) \citep{lee2026flowmaplanguagemodels,hu2026elfembeddedlanguageflows,chen2026langflowcontinuousdiffusionrivals}.
While existing FMLMs use synchronous flows in which all tokens progress together, \ours{}s learn the flow maps of \textit{asynchronous flows}, in which tokenwise times control the relative progress of tokens to specify an ordering or schedule.
Building on prior work \citep{qian2026learning,bansal2026token}, we formulate a general class of asynchronous flows that admits both prescribed and learned, state-adaptive schedules.
By treating tokenwise times as part of the state, we show that these flows define valid flow maps, and develop methods to distill them.
The resulting models generate all tokens in a single forward pass, internalizing the ordering encoded in the flow into their hidden computation.

\begin{figure}
    \centering
    \includegraphics[width=\linewidth]{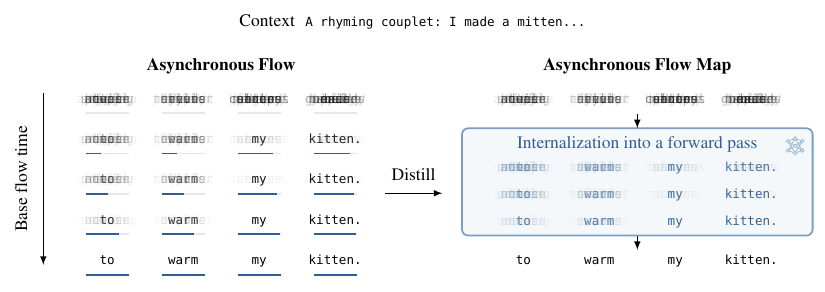}
    \caption{\textbf{Overview.} (Left) We formulate asynchronous flows that denoise tokens at different rates. To rhyme with ``mitten,'' ``kitten'' emerges early and guides the preceding words. (Right) We distill these flows into an asynchronous flow map that internalizes this ordering and generates the full sequence in one forward pass.}
    \label{fig:overview}
\end{figure}

An illustration of our approach is shown in \Cref{fig:overview}.
To summarize, our main contributions are:
\begin{itemize}[leftmargin=10pt,itemsep=2pt]
    \item \textbf{Asynchronous flows.}
    We formulate the generation of asynchronous flow language models (\emph{a}FLMs) as a coupled flow of tokens and their tokenwise times governed by a scheduler, covering prescribed and learned, state-adaptive schedules ported from MDM samplers (\Cref{sec:tokenwise_time_training,sec:state_dependent_async_flows}).
    \item \textbf{Asynchronous flow maps.}
    We show that asynchronous flows define a valid flow map and derive distillation objectives that learn \ours{}s as one-step generators, in which a single model jointly predicts how the tokenwise times and the tokens progress (\Cref{sec:joint_flow_maps,sec:implementation}).
    \item \textbf{One-step generation.}
    We train \ours{}s that achieve state-of-the-art one-step performance on Sudoku solving and coding for GSM8K math problems, and analyze their hidden computations, showing that they internalize the effect of ordering and scheduling into one forward pass (\Cref{sec:experiments}).
\end{itemize}

%% file: sections/related_work.tex
\section{Related Work}
\label{sec:related}

\paragraph{Prediction order in sequence modeling.}
Although every factorization of a joint distribution yields the same likelihood, the quality of sequence models depends on the order of prediction, a gap attributed to their computational limits \citep{finzi2026entropy}.
For ARMs, learning English in the left-to-right order leads to a higher likelihood than in reverse order \citep{shannon1951prediction,papadopoulos2024arrows}, reversing digits makes arithmetic easier to learn \citep{lee2023teaching}, and learning Sudoku solving in the order given by a logical solver outperforms the row-by-row order \citep{shah2024causal}.
In MDMs, adaptive samplers that choose the ordering from model outputs often outperform random ordering \citep{nie2026large,benhamu2025accelerated,fathi2025unifying,kim2026klass}.
MDMs with adaptive samplers have been distilled into fewer steps \citep{chen2025dparallel,kim2025cdlm,liang2026cd4lm}, but along discrete trajectories that make very few-step generation hard \citep{lee2026flowmaplanguagemodels}.

\paragraph{Asynchronous diffusion and flows.}
Asynchronous (continuous) diffusion models denoise different components at different noise levels.
Most prescribe a schedule ahead of generation, such as left-to-right for text \citep{wu2023ar,schiff2026clock,chavan2026large}, sliding window for video and fluid dynamics \citep{ruhe2024rolling}, offsets between image semantics and texture \citep{pan2025semantics,qian2026learning,baade2026latent}, or algorithmic orders for satisfiability problems \citep{bhatt2026generating}.
Diffusion forcing \citep{chen2024diffusion} uses per-token independent noise levels for training so that schedules can be chosen at inference, such as learned ones that depend on model outputs \citep{wewer2025spatialreasoningdenoisingmodels,hu2025asynchronous,schusterbauer2026denoising}.
Few-step distillation of asynchronous models, however, has been limited to prescribed schedules \citep{yin2025causvid,huang2025selfforcing,chavan2026large} or their close variants \citep{bansal2026token}, whereas we consider both prescribed and learned schedules.

\paragraph{Ordering in flow language models.}
Most continuous FLMs evolve all tokens with a single shared time \citep{lee2026flowmaplanguagemodels,chen2026langflowcontinuousdiffusionrivals,hu2026elfembeddedlanguageflows}, and their flow maps compress the resulting trajectories \citep{lee2026flowmaplanguagemodels,roos2026categorical,potaptchik2026discrete}.
Outside prescribed left-to-right ordering \citep{wu2023ar,schiff2026clock,chavan2026large}, order has been introduced into FLMs by finalizing tokens during generation, either stochastically \citep{pynadath2026candihybriddiscretecontinuousdiffusion,jutrasdube2026stickyjumpdiffusionsunifying} or once the model is confident \citep{azangulov2026masked,agarwal2026posteriorrefinementfastlanguage}.
These samplers, however, rely on stochastic updates or restarts with fresh noise, which complicates flow map distillation requiring deterministic trajectories.
\citet{rekesh2026inverse} adjust tokenwise step sizes using model outputs, which our framework can express as a learned scheduler and distill.
\citet{bansal2026token} select among a set of prescribed schedules using model outputs at the first generation step, allowing the use of flow map distillation methods for prescribed schedules.
Our approach goes beyond this workaround and handles fully learned schedules that use all intermediate outputs.

%% file: sections/background.tex
\section{Preliminaries}
\label{sec:problem}

\paragraph{Flow language models.}
Continuous flow language models \citep{lee2026flowmaplanguagemodels, hu2026elfembeddedlanguageflows, chen2026langflowcontinuousdiffusionrivals} model categorical token sequences by learning a continuous flow from noise to continuous token embeddings.
Let $V$ be a token vocabulary and ${\bf y}=({\bf y}^i)_{i=1}^L\in V^L$ a sequence.
Following \citet{lee2026flowmaplanguagemodels}, we represent each token by its one-hot encoding, ${\bf x}_1=(\mathsf{onehot}({\bf y}^i))_{i=1}^L \in\mathbb R^{L\times |V|}$.
Let $p_1$ denote the distribution of these one-hot representations and $p_0=\mathsf N(0,I)$ the noise distribution.
We define a probability path by interpolating between
independent draws ${\bf x}_0\sim p_0$ and
${\bf x}_1\sim p_1$:
\begin{equation}
    I_t\coloneqq(1-t){\bf x}_0+t{\bf x}_1, \qquad p_t \coloneqq \mathsf{Law}(I_t),\qquad t\in[0,1].
    \label{eq:sync_interpolant}
\end{equation}
The probability path admits a deterministic representation,
\begin{equation}\label{eq:sync_flow}
    \dot{\bf x}_t = b_t({\bf x}_t),\qquad b_t({\bf x}) = \mathbb{E}\left[\frac{{\rm d}}{{\rm d}t}I_t \middle| I_t={\bf x}\right] = \mathbb{E}[{\bf x}_1 - {\bf x}_0|I_t={\bf x}],
\end{equation}
where $b_t$ is the velocity field of the probability flow.
Instead of predicting the velocity itself, it is common to learn the denoiser, which predicts the posterior mean of the clean data:
\begin{equation}
    D_t({\bf x})\coloneqq \mathbb{E}[{\bf x}_1|I_t={\bf x}],\qquad b_t({\bf x}) = \frac{D_t({\bf x}) - {\bf x}}{1-t}.
    \label{eq:denoiser}
\end{equation}
Since ${\bf x}_1$ is one-hot, $D_t^i({\bf x})$ is a probability distribution over $V$ at each position $i$, and the denoiser can be learned by minimizing a tokenwise cross-entropy loss $\mathcal L_\mathsf{CE}(\hat D)\coloneqq \mathbb{E}_{t,{\bf x}_0,{\bf x}_1}\sum_i[-\log \hat D_t^i(I_t)]_{{\bf y}^i}$ \citep{dieleman2022continuous,eijkelboom2024variational,lee2026flowmaplanguagemodels}.

\paragraph{Flow map language models.}
The deterministic flow given in \Cref{eq:sync_flow} defines a unique flow map $X_{s,t}: \mathbb{R}^{L\times |V|} \rightarrow \mathbb{R}^{L\times |V|}$, the solution operator of the flow ${\bf x}_t=X_{s,t}({\bf x}_s)$, satisfying
\begin{equation} X_{s,t}({\bf x}) = {\bf x} + \int_s^t b_r(X_{s,r}({\bf x})) {\rm d}r = {\bf x} + (t-s)v_{s,t}({\bf x}), 
\end{equation} where $v_{s,t}$ is the average velocity between $s$ and $t$, satisfying $v_{s,s}=b_s$.
By learning the flow map, we can jump between arbitrary times, and evaluating it with $s=0$ and $t=1$ gives a one-step jump from noise to data.
Recent work on learning flow maps over categorical data \citep{roos2026categorical,lee2026flowmaplanguagemodels,potaptchik2026discrete} parameterizes it via the two-time denoiser $\delta_{s,t}:\mathbb{R}^{L\times |V|}\rightarrow(\Delta^{|V|-1})^L$,
\begin{equation}
    \delta_{s,t}({\bf x}) \coloneqq {\bf x} + (1-s)v_{s,t}({\bf x}),
    \label{eq:two_time_denoiser}
\end{equation}
which extrapolates the flow map displacement $s\to t$ to $s\to 1$ and satisfies $\delta_{s,s}({\bf x})=D_s({\bf x})$. 
Since it is simplex-valued, $\delta_{s,t}$ can be learned by self-distillation with KL divergence \citep{lee2026flowmaplanguagemodels},
\begin{equation}
    \mathcal L_\mathsf{KL}(\hat\delta)=
    \mathbb E_{t,{\bf x}_0,{\bf x}_1}\left[\mathsf{KL}(D_t(I_t)\,\|\,\hat\delta_{t,t}(I_t))\right]
    +\mathbb E_{s,u,t,{\bf x}_0,{\bf x}_1}\left[\mathsf{KL}(\bar\delta_{s,t}\,\|\,\hat\delta_{s,t}(I_s))\right],
    \label{eq:fmlm_loss}
\end{equation}
where, for $s<u<t$ and with the stop-gradient operator $\mathsf{sg}(\cdot)$,
\begin{equation}
    \bar\delta_{s,t}=\mathsf{sg}(\gamma\,\hat\delta_{s,u}(I_s)+(1-\gamma)\,\hat\delta_{u,t}(\hat{X}_{s,u}(I_s))),
    \qquad
    \gamma=\textstyle{\frac{(1-t)(u-s)}{(1-u)(t-s)}}\in [0,1],
\end{equation}
and we denote by $\hat{X}_{s,t}({\bf x}) = \frac{1-t}{1-s}{\bf x}+\frac{t-s}{1-s}\hat{\delta}_{s,t}({\bf x})$ the flow map induced by the learned denoiser $\hat{\delta}$.

%% file: sections/method.tex
\section{Asynchronous Flow Maps}
\label{sec:method}

We first formulate asynchronous flows under a scheduler, then distill them into asynchronous flow maps for one-step generation.

\subsection{Asynchronous flows with schedulers}
\label{sec:state_dependent_async_flows}
\label{sec:tokenwise_time_training}

To allow token positions to denoise at different rates, we assign each position its own time:
\begin{equation}
    \tau=(\tau^1,\ldots,\tau^L)\in[0,1]^L,
    \label{eq:tokenwise_time}
\end{equation}
where $\tau^i$ measures the denoising progress of position $i$ \citep{chen2024diffusion}.
Extending the standard interpolant \Cref{eq:sync_interpolant}, we associate $\tau={\bf0}$ with noise and $\tau={\bf1}$ with clean data.
For independent draws ${\bf x}_{\bf0}\sim p_{\bf0}\coloneqq p_0$ and ${\bf x}_{\bf1}\sim p_{\bf1}\coloneqq p_1$, we define the \emph{asynchronous interpolant} \citep{qian2026learning}
\begin{equation}
    I_\tau\coloneqq
    (\mathbf1-\tau){\bf x}_{\bf0}+\tau{\bf x}_{\bf1},
    \qquad
    p_\tau\coloneqq\mathsf{Law}(I_\tau),
    \label{eq:df_interpolant}
\end{equation}
where we write products and divisions of tokenwise vectors as $(\tau{\bf x})^i=\tau^i{\bf x}^i$ and $({\bf x}/\tau)^i={\bf x}^i/\tau^i$.

\paragraph{Schedules.}
In standard flows \Cref{eq:sync_flow}, generation moves from noise to data as a shared time advances from zero to one, and we only need to choose step sizes to perform generation.
With tokenwise times \Cref{eq:tokenwise_time}, we can think of generation as the analogous transition of $\tau\in [0,1]^L$ from ${\bf0}$ to ${\bf1}$.
Yet, unlike in standard flows, this further leaves a choice of how the tokens progress, as some may advance earlier than others, or all may advance together.
These choices give many distinct paths of generation:
\begin{restatable}{proposition}{monotone}
\label{prop:monotone}
Up to continuous nondecreasing reparameterization, there is a unique continuous nondecreasing path from zero to one.
In contrast, there are infinitely many such paths from $\mathbf0$ to $\mathbf1$ in $[0,1]^L$ for $L\geq2$.
\end{restatable}
A proof is in \Cref{app:monotone}.

To describe the progress of generation along such a path, we introduce a scalar \emph{base time} $t\in[0,1]$ and call the resulting evolution of tokenwise times $\tau_t\in[0,1]^L$ from $\tau_0={\bf0}$ to $\tau_1={\bf1}$ a \emph{schedule}.
We specify this evolution through a \emph{scheduler} $\omega_t:\mathbb R^{L\times|V|}\times[0,1]^L \to\mathbb R_+^L$ that sets the tokenwise rates:
\begin{equation}
    \dot\tau_t=\omega_t({\bf x}_t,\tau_t),
    \qquad \tau_0=\mathbf0,
    \label{eq:state_dependent_schedule}
\end{equation}
where ${\bf x}_t$ denotes the current token states.
We choose the scheduler so that $\tau_t\to\mathbf1$ as $t\uparrow1$.

\paragraph{Velocity of token states.}
We now derive the velocity of token states under a chosen scheduler by extending \Cref{eq:sync_flow}.
For each pair of noise ${\bf x}_{\bf0}$ and data ${\bf x}_{\bf1}$, we advance the tokenwise times by evaluating the scheduler on the interpolant, $\dot\tau_t=\omega_t(I_{\tau_t},\tau_t)$.
Since $\frac{{\rm d}}{{\rm d}t}I_{\tau_t}=\dot\tau_t({\bf x}_{\bf1}-{\bf x}_{\bf0})$, the conditional mean
\begin{equation}
    b^\omega_t({\bf x},\tau)
    \coloneqq
    \mathbb E[{\bf x}_{\bf1}-{\bf x}_{\bf0}
    \mid I_{\tau_t}={\bf x},\,\tau_t=\tau]
    \label{eq:exact_velocity}
\end{equation}
gives the \emph{asynchronous velocity} of token states, $\dot\tau_t b^\omega_t = \omega_t b^\omega_t$.
Together with the scheduler, it defines a valid probability flow from noise to data, which we refer to as an \emph{asynchronous flow}:
\begin{restatable}{proposition}{exactflow}
\label{prop:exact_flow}
Let $({\bf x}_t,\tau_t)$ solve, starting from ${\bf x}_0\sim p_{\bf 0}$ and $\tau_0={\bf 0}$,
\begin{equation*}
    \dot{\bf x}_t=\omega_t({\bf x}_t,\tau_t)\,b^\omega_t({\bf x}_t,\tau_t),\qquad
    \dot\tau_t=\omega_t({\bf x}_t,\tau_t).
\end{equation*}
Then $({\bf x}_t,\tau_t)$ has the same joint distribution as the interpolant and its tokenwise times advanced by $\omega$ for every $t\in[0,1)$, and ${\bf x}_t$ converges in distribution to $p_{\bf1}$ as $t\uparrow1$.
\end{restatable}
A proof is in \Cref{app:exact_flow}.

Mirroring \Cref{eq:denoiser}, we can parameterize the velocity through the corresponding \emph{asynchronous denoiser}:
\begin{equation}
    D^\omega_t({\bf x},\tau)
    \coloneqq
    \mathbb E[{\bf x}_{\bf1}
    \mid I_{\tau_t}={\bf x},\,\tau_t=\tau],
    \qquad
    b^\omega_t({\bf x},\tau)
    =\frac{D^\omega_t({\bf x},\tau)-{\bf x}}
    {\mathbf1-\tau}.
\label{eq:scheduler_denoiser}
\end{equation}

\paragraph{Scheduler-independent training.}
For a chosen scheduler $\omega$, we can learn $D^\omega_t$ by tokenwise cross-entropy, using $(t, I_{\tau_t},\tau_t)$ as inputs and the clean tokens as targets.
Yet, constructing these examples requires simulating $\omega$ along each interpolant, tying denoiser training to the scheduler.
To allow the scheduler to be chosen after training, we may instead learn the \emph{scheduler-independent denoiser}
\begin{equation*}
    D_\tau({\bf x})
    \coloneqq
    \mathbb E[{\bf x}_{\bf1}\mid I_\tau={\bf x}],
\end{equation*}
which takes as input $(\tau, I_\tau)$.
For this, we use diffusion forcing~\citep{chen2024diffusion}, which samples the tokenwise times independently of one another and of the noise and data:
\begin{restatable}{proposition}{dfoptimum}
\label{prop:df_optimum}
Consider the objective
\begin{equation}
    \mathcal L_\mathsf{CE}(\hat D)
    \coloneqq
    \mathbb E_{\tau,{\bf x}_{\bf0},{\bf x}_{\bf1}}
    \sum_{i=1}^L
    \left[-\log\hat D_\tau^i(I_\tau)\right]_{{\bf y}^i},
    \qquad
    \tau^i\overset{\textnormal{iid}}{\sim}\mathsf U[0,1].
    \label{eq:df_loss}
\end{equation}
Any population minimizer $\hat D^\star$ satisfies
$\hat D^\star_\tau=D_\tau$ almost surely under the training distribution.
\end{restatable}
A proof is in \Cref{app:df_optimum}.
We leave further discussion on this training scheme in \Cref{app:df_support}.

Following \Cref{eq:denoiser}, the denoiser $D_\tau$ gives the scheduler-independent velocity
\begin{equation}
    b_\tau({\bf x})
    \coloneqq
    \mathbb E[{\bf x}_{\bf1}-{\bf x}_{\bf0}\mid I_\tau={\bf x}]
    =
    \frac{D_\tau({\bf x})-{\bf x}}{\mathbf1-\tau},
    \label{eq:df_velocity}
\end{equation}
which we use to approximate $b^\omega_t$ at inference under a choice of scheduler $\omega$.
Denoting ${\bf z}_t\coloneqq({\bf x}_t,\tau_t)$, this gives the asynchronous flow
\begin{equation}
    \dot{\bf z}_t =
    \big(
        \omega_t({\bf x}_t,\tau_t)b_{\tau_t}({\bf x}_t),
        \ \omega_t({\bf x}_t,\tau_t)
    \big),
    \qquad {\bf z}_0=({\bf x}_0,\mathbf0),
    \qquad {\bf x}_0\sim p_{\bf0}.
\label{eq:state_dependent_dynamics}
\end{equation}
These coupled dynamics define a deterministic ODE on the token states and their tokenwise times.
We call the resulting model an asynchronous flow language model (\emph{a}FLM).
We note that setting $\tau_t^i=t$ for all positions $i$ recovers the usual FLM, which we refer to as synchronous.

\paragraph{Prescribed and state-adaptive schedules.}
When $\omega$ does not depend on the token states, as in left-to-right generation for instance, $\tau_t$ follows a fixed path $\tau(t)$ determined ahead of generation, which we call a \emph{prescribed schedule}.
The condition $\tau_t=\tau(t)$ then holds for every pair of noise and data, so $D^\omega_t({\bf x},\tau(t))=D_{\tau(t)}({\bf x})$.
The velocity approximation is therefore exact:
\begin{restatable}{corollary}{asyncflow}
\label{prop:async_flow}
Let $\tau(t)$ be a prescribed schedule.
Then the asynchronous flow in \Cref{eq:state_dependent_dynamics} reduces to
\begin{equation}
    \dot{\bf x}_t=\dot\tau(t)\,b_{\tau(t)}({\bf x}_t),
    \qquad {\bf x}_0\sim p_{\bf0},
    \label{eq:async_flow}
\end{equation}
and satisfies ${\bf x}_t\sim p_{\tau(t)}$ for every $t\in[0,1)$,
with $p_{\tau(t)}\to p_{\bf1}$ as $t\uparrow1$.
\end{restatable}
A proof is in \Cref{app:async_flow}.
Furthermore, only the denoiser along $\tau(t)$ is needed, so we can train it with the cross-entropy loss in \Cref{eq:df_loss} using $\tau=\tau(t)$ instead of independently sampled tokenwise times.

Schedulers that depend on the current token states yield \emph{state-adaptive} schedules, and we call those computed from denoiser predictions \emph{learned schedules}.
Then $D^\omega_t$ and $D_\tau$ generally differ, and exact recovery of $p_{\bf1}$ is no longer guaranteed.
Nevertheless, this allows training the denoiser without fixing a scheduler in advance.
Then during inference, a scheduler can be chosen to improve generation by presenting an imperfect denoiser with easier prediction problems~\citep{kim2025train}.
We construct such schedulers as continuous analogues of unmasking criteria used in MDMs, as given in \Cref{sec:implementation}.

\subsection{Asynchronous flow maps}
\label{sec:joint_flow_maps}

Asynchronous flows exploit prediction order by letting tokens that denoise earlier inform the prediction of others.
This motivates our key idea for leveraging such an order for few-step and one-step parallel generation, by distilling the asynchronous flow into a flow map.
In this way, the effect of the schedule is internalized within the model's hidden computation.
Since the flow is defined jointly on the tokens and their tokenwise times \Cref{eq:state_dependent_dynamics}, we define the \emph{asynchronous flow map} to evolve both:
\begin{restatable}{proposition}{stateasyncflowmap}
\label{prop:state_async_flow_map}
On any interval where the asynchronous flow \Cref{eq:state_dependent_dynamics} admits unique solutions, its solution operator
$F_{s,t}({\bf z}_s)={\bf z}_t$, called the \emph{asynchronous flow map}, satisfies
\begin{equation}
\begin{aligned}
    F_{s,t}({\bf z}_s)
    &={\bf z}_s+\int_s^t\dot{\bf z}_r\,{\rm d}r,\\
    F_{s,t}
    &=F_{u,t}\circ F_{s,u},
    \qquad s\leq u\leq t<1.
\end{aligned}
\label{eq:state_semigroup_condition}
\end{equation}
\end{restatable}
A proof is in \Cref{app:state_async_flow_map}.
The composition identity, which we call the semigroup condition, states that a direct jump from $s$ to $t$ agrees with two successive jumps through $u$.
We use this consistency to learn the map by self-distillation.

\paragraph{Parameterization.}
Following \Cref{sec:problem}, we express the flow map through its average velocity.
For a state ${\bf z}=({\bf x},\tau)$ at time $s$, we denote the average rates of change of the tokens and their tokenwise times until time $t$ by $v_{s,t}$ and $\eta_{s,t}$, respectively, with $s=t$ reducing to the instantaneous dynamics:
\begin{equation}
\begin{aligned}
    \big(v_{s,t}({\bf z}),\eta_{s,t}({\bf z})\big)
    &\coloneqq
    \frac{F_{s,t}({\bf z})-{\bf z}}{t-s},
    \qquad s<t,\\
    \big(v_{t,t}({\bf z}),\eta_{t,t}({\bf z})\big)
    &\coloneqq
    \big(
        \omega_t({\bf x},\tau)b_\tau({\bf x}),
        \omega_t({\bf x},\tau)
    \big).
\end{aligned}
\label{eq:joint_average_velocities}
\end{equation}
For the token component $v_{s,t}$, we further adapt a two-time denoiser parameterization:
\begin{equation}
    \delta_{s,t}({\bf z})
    \coloneqq
    {\bf x}
    +\frac{\mathbf1-\tau}{\eta_{s,t}({\bf z})}
    v_{s,t}({\bf z}).
    \label{eq:state_two_time_denoiser}
\end{equation}
This extends \Cref{eq:two_time_denoiser} by rescaling each token's displacement according to its tokenwise progress.
Along the asynchronous flow, $\delta_{s,t}^i$ is a weighted average of denoiser predictions and therefore lies in the vocabulary simplex.
We can thus learn it with a categorical output and KL loss, as in FMLMs.
The following identities recover the flow map and provide the conditions for learning it:
\begin{restatable}{proposition}{statetwotimedenoiser}
\label{prop:state_two_time_denoiser}
Let $F_{s,t}$ be the asynchronous flow map.
Then $\delta_{s,t}$ and $\eta_{s,t}$ satisfy:
\begin{enumerate}[label=(\roman*),leftmargin=*]
    \item The transition $F_{s,t}({\bf z})=({\bf x}_t,\tau_t)$ is recovered as
    \begin{equation}
    \begin{aligned}
        {\bf x}_t
        &={\bf x}
        +\frac{(t-s)\eta_{s,t}({\bf z})}{\mathbf1-\tau}
        \big(\delta_{s,t}({\bf z})-{\bf x}\big),\\
        \tau_t
        &=\tau+(t-s)\eta_{s,t}({\bf z}).
    \end{aligned}
    \label{eq:joint_student_map}
    \end{equation}

    \item On the diagonal $s=t$,
    \begin{equation}
        \delta_{t,t}({\bf z})=D_\tau({\bf x}),
        \qquad
        \eta_{t,t}({\bf z})=\omega_t({\bf x},\tau).
        \label{eq:state_denoiser_diagonal}
    \end{equation}

    \item For $s<u<t$, the semigroup condition
    \Cref{eq:state_semigroup_condition} is equivalent to
    \begin{equation}
    \begin{aligned}
        \delta_{s,t}({\bf z})
        &=\gamma\delta_{s,u}({\bf z})
        +(\mathbf1-\gamma)\delta_{u,t}(F_{s,u}({\bf z})),\\
        \eta_{s,t}({\bf z})
        &=\frac{u-s}{t-s}\eta_{s,u}({\bf z})
        +\frac{t-u}{t-s}\eta_{u,t}(F_{s,u}({\bf z})),
    \end{aligned}
    \label{eq:state_denoiser_composition}
    \end{equation}
    where
    \[
        \gamma
        =\frac{(\mathbf1-\tau_t)(\tau_u-\tau_s)}
        {(\mathbf1-\tau_u)(\tau_t-\tau_s)}
        \in[0,1]^L,
    \]
    with $\tau_s=\tau$ and $\tau_u,\tau_t$ the tokenwise times along the flow from ${\bf z}$.
\end{enumerate}
\end{restatable}
A proof is in \Cref{app:state_two_time_denoiser}.

\paragraph{Training.}
We learn predictors $\hat\delta$ and $\hat\eta$, which define a model of the flow map $\hat F$ through \Cref{eq:joint_student_map}.
On the diagonal, the denoiser $D_\tau$ and scheduler $\omega$ provide their respective targets \Cref{eq:state_denoiser_diagonal}.
For $s<u<t$, we construct targets for a direct transition from $s$ to $t$ by composing two shorter student transitions through $u$ based on \Cref{eq:state_denoiser_composition}:
\begin{equation}
\begin{aligned}
    \mathcal T^\delta_{s,u,t}({\bf z})
    &\coloneqq
    \gamma\hat\delta_{s,u}({\bf z})
    +(\mathbf1-\gamma)
    \hat\delta_{u,t}(\hat F_{s,u}({\bf z})),\\
    \mathcal T^\eta_{s,u,t}({\bf z})
    &\coloneqq
    \frac{u-s}{t-s}\hat\eta_{s,u}({\bf z})
    +\frac{t-u}{t-s}\hat\eta_{u,t}(\hat F_{s,u}({\bf z})).
\end{aligned}
\label{eq:state_composition_targets}
\end{equation}
Here $\gamma$ from \Cref{eq:state_denoiser_composition} is evaluated at the tokenwise times of
${\bf z}$, $\hat F_{s,u}({\bf z})$, and $\hat F_{u,t}(\hat F_{s,u}({\bf z}))$.
Combining these composition targets with the diagonal targets gives the training objective:
\begin{restatable}{proposition}{stateLossOptimum}
\label{prop:state_loss_optimum}
For $\lambda>0$, consider the population objective
\begin{equation}
\begin{aligned}
    \mathcal L(\hat\delta,\hat\eta)
    &=
    \mathbb E_{s,u,t,{\bf z}}
    \big[
        \mathsf{KL}\big(
            \mathsf{sg}(\mathcal T^\delta_{s,u,t}({\bf z}))
            \,\|\,\hat\delta_{s,t}({\bf z})
        \big)
    \big] +
    \mathbb E_{t,{\bf z}}
    \big[
        \mathsf{KL}\big(
            D_\tau({\bf x})
            \,\|\,\hat\delta_{t,t}({\bf z})
        \big)
    \big]\\
    &+
    \lambda\mathbb E_{s,u,t,{\bf z}}
    \big\|
        \mathsf{sg}(\mathcal T^\eta_{s,u,t}({\bf z}))
        -\hat\eta_{s,t}({\bf z})
    \big\|_2^2 +
    \lambda\mathbb E_{t,{\bf z}}
    \big\|
        \omega_t({\bf x},\tau)
        -\hat\eta_{t,t}({\bf z})
    \big\|_2^2.
\end{aligned}
\label{eq:joint_self_distillation_loss}
\end{equation}
The exact $\delta$ and $\eta$ have zero loss.
Conversely, any zero-loss solution satisfies the diagonal and composition identities
\Cref{eq:state_denoiser_diagonal,eq:state_denoiser_composition}
almost surely under the training distribution.
\end{restatable}
A proof is in \Cref{app:state_loss_optimum}.
We call the model learned with \Cref{eq:joint_self_distillation_loss}
an asynchronous flow map language model (\ours{}).

For state-adaptive schedules, we train on joint states ${\bf z}_s$ visited by the asynchronous flow.
These states depend on the preceding trajectory and cannot be sampled directly from the interpolant.
We therefore first train the denoiser $D_\tau$ and simulate its flow under the chosen scheduler to obtain the initial training states.
This pretrained teacher also provides the diagonal denoising target in \Cref{eq:joint_self_distillation_loss}.
To reduce the cost of simulation, we can later use the learned map to get ${\bf z}_s$, as described in \Cref{sec:implementation}.

\label{sec:prescribed_async_flows}
On the other hand, for a prescribed schedule $\tau(t)$, the time component is fixed:
\begin{equation*}
    \eta_{s,t}
    =\frac{\tau(t)-\tau(s)}{t-s},
    \qquad
    \eta_{t,t}=\dot\tau(t).
\end{equation*}
Only the two-time denoiser needs to be learned.
Moreover, the flow follows $p_{\tau(t)}$ by \Cref{prop:async_flow}, so a training state at any base time $r$ can be sampled directly as ${\bf z}_r=(I_{\tau(r)},\tau(r))$.
On the diagonal, the denoiser is the conditional mean of the clean data ${\bf x}_{\bf1}$.
Replacing its target with data therefore changes the expected loss only by a constant independent of the model:
\begin{restatable}{corollary}{prescribedLossOptimum}
\label{prop:prescribed_loss_optimum}
Let $\tau(t)$ be a prescribed schedule, with $\hat\eta$ fixed to its exact time component and training states sampled from the corresponding interpolants.
Then \Cref{eq:joint_self_distillation_loss} reduces to
\begin{equation}
    \mathcal L(\hat\delta)
    =
    \mathbb E_{s,u,t,{\bf z}}
    \big[
        \mathsf{KL}\big(
            \mathsf{sg}(\mathcal T^\delta_{s,u,t}({\bf z}))
            \,\|\,\hat\delta_{s,t}({\bf z})
        \big)
    \big]
    +
    \mathbb E_{t,{\bf z},{\bf x}_{\bf1}}
    \big[
        \mathsf{KL}\big(
            {\bf x}_{\bf1}
            \,\|\,\hat\delta_{t,t}({\bf z})
        \big)
    \big],
\label{eq:prescribed_self_distillation_loss}
\end{equation}
up to a constant independent of $\hat\delta$.
Its minimizers coincide with the zero-loss solutions of
\Cref{prop:state_loss_optimum} with $\hat\eta$ fixed by the schedule.
\end{restatable}
A proof is in \Cref{app:prescribed_distillation}.
Thus, a prescribed schedule allows the flow map to be trained from scratch without a pretrained teacher.

\paragraph{Inference.}
We generate samples by applying the learned flow map on a temporal grid of base times $0=t_0<\cdots<t_N=1$:
\begin{equation}
\begin{aligned}
    \hat{\bf z}_{t_0}
    &=({\bf x}_0,\mathbf0),
    \qquad {\bf x}_0\sim p_{\bf0},\\
    \hat{\bf z}_{t_{n+1}}
    &=\hat F_{t_n,t_{n+1}}(\hat{\bf z}_{t_n}),
    \qquad n=0,\ldots,N-1.
\end{aligned}
\end{equation}
The final tokens are obtained by applying tokenwise $\mathsf{argmax}$ to
$\hat\delta_{t_{N-1},1}(\hat{\bf z}_{t_{N-1}})$.
For $N=1$, a single evaluation of $\hat\delta_{0,1}$ generates the entire sequence without evaluating the scheduler.
For supervised tasks, this allows us to train with a prescribed schedule for each example, even when the corresponding schedule is unavailable for a test input.
These schedules guide training, and the one-step model internalizes them to generate its output from the task input and initial noise alone.

\subsection{Algorithmic details}
\label{sec:implementation}

\paragraph{Scheduler implementation.}
We remark that a scheduler must send $\tau_t\to {\bf1}$ as $t\uparrow 1$ \Cref{eq:state_dependent_schedule}.
We implement the schedulers in a common form that ensures this by construction (\Cref{app:schedule_construction}):
\begin{equation}
    \omega_t^i({\bf x},\tau)
    =\frac{1-\tau^i}{1-t}\,k_t^i({\bf x},\tau),
    \label{eq:common_schedule_rate}
\end{equation}
where the factor $k_t^i>0$ controls how quickly position $i$ advances and can be chosen almost freely.
Setting $k_t^i=1$, for instance, recovers the synchronous schedule $\tau_t^i=t$ from $\tau_0=\mathbf0$.
For learned schedules, we assign higher factors to positions that the denoiser predicts confidently, using two rules:
\begin{itemize}[leftmargin=*]
    \item \textbf{Top-$K$ confidence scheduler}
    prioritizes the $K$ most confident positions, with $K$ increasing over time~\citep{chang2022maskgit}.
    \item \textbf{Entropy-bounded scheduler}
    uses an entropy budget to determine how many high-confidence positions advance faster, allowing more positions to advance together when their predictions are certain~\citep{benhamu2025accelerated}.
\end{itemize}
These schedules are constructed as continuous analogues of unmasking strategies used in MDMs, and their details along with more such examples can be found in \Cref{app:schedule_construction}.

\paragraph{Denoiser and time heads.}
For flow map distillation under learned schedules, we use one model with a denoiser head for $\hat\delta_{s,t}$ and a time head for $\hat\eta_{s,t}$.
To keep predicted tokenwise times in $[\tau_s,\mathbf1]$, we parameterize the time head via the effective rate $\kappa_{s,t}\coloneqq \log\frac{\mathbf1-\tau_s}{\mathbf1-\tau_t}/ \log\frac{1-s}{1-t}$, which is a weighted average of $k$ along the trajectory
(see \Cref{app:effective_tokenwise_time_rate} for details).
The head outputs $\hat\kappa_{s,t}>0$, predicting
\begin{equation}
    \hat{\tau}_t = \tau_s+(t-s)\hat\eta_{s,t}({\bf z}_s)
    =
    \mathbf1-(\mathbf1-\tau_s)\odot
    \left(\frac{1-t}{1-s}\right)^{\hat\kappa_{s,t}({\bf z}_s)}.
    \label{eq:joint_clock_parameterization}
\end{equation}
Each transition advances the tokenwise times toward $\mathbf1$ and reaches $\mathbf1$ at $t=1$.
On the diagonal, matching $\hat\kappa_{t,t}$ to $k_t$ is equivalent to matching $\hat\eta_{t,t}$ to $\omega_t$
(see \Cref{app:effective_tokenwise_time_rate} for details).

We initialize the backbone and denoiser head from the teacher $D_\tau$ and attach the time head to the final hidden features.
The time head exponentiates its output to ensure positivity, and we initialize its output before exponentiation to zero, setting $\hat\kappa_{s,t}={\bf1}$ at initialization.

\paragraph{Time conditioning.}
With tokenwise time, both the teacher denoiser and the flow map receive a separate time embedding at each position, obtained by applying a shared embedding layer to $\tau^i$.
For the flow map, we additionally embed the base times $s$ and $t$ with zero-initialized layers and add their outputs to the embeddings of $\tau_s$, following \citet{lee2026flowmaplanguagemodels}.
The resulting conditioning vectors modulate each block through adaptive layer normalization (adaLN)~\citep{peebles2023scalable}.

\paragraph{Rollout curriculum.}
To train with \Cref{eq:joint_self_distillation_loss}, we initially obtain training states ${\bf z}_s$ by rolling out the teacher under the chosen scheduler from ${\bf z}_0=({\bf x}_0,\mathbf0)$ to time $s$.
After the first round, we use the current student flow map $\hat F$ for rollouts and halve the number of rollout steps at successive rounds.
The fixed teacher denoiser and scheduler continue to provide the diagonal targets throughout training.

\paragraph{MDM initialization.}
Diffusion forcing requires the denoiser to denoise across many combinations of tokenwise noise levels, which can make training demanding~\citep{kim2025train}.
When each tokenwise time is close to either $0$ or $1$, the task resembles masked token prediction, since each position is nearly pure noise or nearly clean.
To exploit this overlap, we optionally initialize the denoiser from a trained MDLM~\citep{sahoo2024simple}, following \citet{agarwal2026posteriorrefinementfastlanguage}.

\paragraph{Noise downscaling.}
Following \citet{chavan2026large}, we use initial noise ${\bf x}_{\bf0}\sim\mathsf N(0,\sigma^2I)$ with $\sigma<1$ during flow map distillation and inference.
In supervised tasks where one accurate answer is sufficient, reducing variation in the initial states can ease the burden of fitting the flow map.
We apply this on GSM8K, where it improves one-step accuracy (see \Cref{fig:gsm8k_teacher_ablation}).

%% file: sections/experiments.tex
\section{Experiments}
\label{sec:experiments}
Throughout the experiments, we address these key questions:
\begin{enumerate}[leftmargin=*,label=\textbf{Q\arabic*.}]
    \setlength{\itemsep}{2pt}
    \setlength{\parskip}{0pt}
    \item Can a proper choice of asynchronous scheduler improve the generation quality of \emph{a}FLMs?
\item Does this improvement transfer to one-step generation with \ours{} after distillation?
    \item Do hidden computations of \ours{}s reflect the schedules of \emph{a}FLMs they are distilled from?
\end{enumerate}

\paragraph{Setup.}
We evaluate our method on Sudoku solving and code generation for GSM8K math word problems.
For Sudoku, we use the dataset with three difficulty levels and report exact solving accuracy (\%) on the full test set, following~\citet{deschenaux2026language}.
For code generation, we train on code completions from TinyGSM \citep{liu2023tinygsm} and measure zero-shot accuracy (\%) on the full test set of GSM8K \citep{cobbe2021training} math problems, following~\citet{kim2025train}.

For the baselines, we consider 
discrete methods MDLM \citep{sahoo2024simple}, Duo \citep{sahoo2025diffusionduality} and AlphaDLM \citep{gushchin2026alphadlm}, and continuous methods CANDI \citep{pynadath2026candihybriddiscretecontinuousdiffusion}, FLM \citep{lee2026flowmaplanguagemodels}, S-FLM \citep{deschenaux2026language}, FMLM \citep{lee2026flowmaplanguagemodels}, FMLM+ \citep{agarwal2026posteriorrefinementfastlanguage}, and DBTM \citep{tang2026discretebeckmanntransportmodels}.
We train and distill for 250k steps on GSM8K and 20k on Sudoku in total, matching the baselines.
We leave further details in \Cref{app:exp_details}.

\subsection{Asynchronous schedulers for \texorpdfstring{$a$}{a}FLM (Q1)}
\label{sec:exp_scheduling_rules}

\begin{figure}[t]
    \centering
    \vspace{-1em}
    \includegraphics[width=\textwidth]{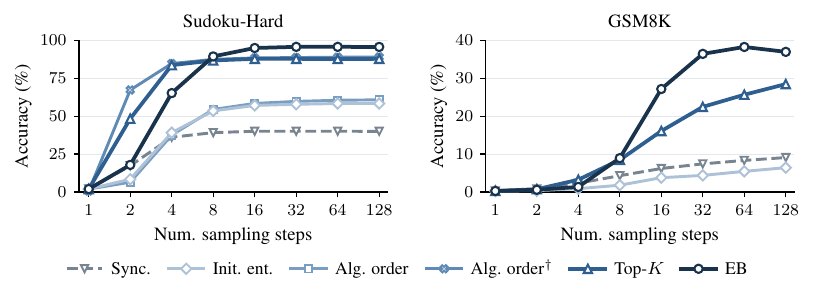}
    \vspace{-2.5em}
    \caption{Accuracy of \emph{a}FLM under different inference-time schedulers on Sudoku-Hard (left) and GSM8K (right). A well-chosen scheduler improves performance over the synchronous one (Sync.). Alg. order$^\dagger$ denotes the \emph{a}FLM whose denoiser is trained only along the algorithmic order.}
    \label{fig:scheduled_inference}
\end{figure}

For each dataset, we train one asynchronous denoiser with diffusion forcing \citep{chen2024diffusion} and apply various schedulers during inference.
We compare the synchronous schedule (Sync.) of FLM \citep{lee2026flowmaplanguagemodels}, a schedule fixed by the initial prediction entropy ({Init. ent.}) of \citet{bansal2026token}, and learned schedulers including the top-$K$ confidence scheduler ({Top-$K$}) and the entropy-bounded scheduler ({EB}) defined in \Cref{sec:implementation}. 
To test whether generation can also benefit from a prescribed order, we include an algorithmic order (Alg. order) on Sudoku. In this order, the cells that a constraint solver resolves earlier are denoised faster \citep{shah2024causal} (\Cref{app:sudoku_order}).

The results are in \Cref{fig:scheduled_inference}, showing that learned schedulers substantially improve generation accuracy on both tasks.
With 32-step generation, the EB scheduler improves the accuracy over Sync. from 40.1\% to 95.6\% on Sudoku-Hard and from 7.4\% to 36.4\% on GSM8K.
Also, schedulers that adapt the order to intermediate predictions (Top-$K$ and EB) outperform the scheduler that fixes the order from the initial prediction (Init. ent.).
Asynchronous flows thus let flow language models benefit from a good generation order, with the largest gains from learned schedulers ported from MDMs.

While the prescribed algorithmic order improves accuracy on Sudoku-Hard from 40.1\% to 59.7\%, it still falls short of the learned schedulers.
We conjecture that this gap arises because diffusion forcing cannot learn all combinations of tokenwise times equally well \citep{kim2025train} (\Cref{sec:implementation}).
Learned schedulers follow the confidence of the denoiser and thus avoid the poorly learned combinations, whereas the algorithmic order may still visit them.
Indeed, training the denoiser only along the algorithmic order raises the accuracy from 59.7\% to 88.6\% (Alg. order$^\dagger$ in \Cref{fig:scheduled_inference}).

\subsection{One-step generation with \texorpdfstring{$a$}{a}FMLM (Q2)}
\label{sec:exp_flow_maps}

\begin{table}[t!]
    \centering
    \vspace{-1em}
    \caption{One-step exact solving accuracy (\%) on Sudoku and zero-shot accuracy (\%) on GSM8K (see \Cref{app:baseline_details} for baseline details).
    For \emph{a}FLM, the schedule does not matter in one-step generation.
    }
    \label{tab:main_results_split}
    \footnotesize
    \setlength{\tabcolsep}{4pt}
    \renewcommand{\arraystretch}{1.0}
    \begin{tabular}{@{}llcccc@{}}
        \toprule
        & & \multicolumn{3}{c}{Sudoku} & GSM8K \\
        \cmidrule(lr){3-5}\cmidrule(l){6-6}
        \multicolumn{2}{@{}l}{Model} & Easy & Medium & Hard & Accuracy \\
        \midrule
        \multicolumn{2}{@{}l}{MDLM \citep{sahoo2024simple}} & 63.6 & 16.2 & 1.7 & 0.0 \\
        \multicolumn{2}{@{}l}{Duo \citep{sahoo2025diffusionduality}} & 69.8 & 24.6 & 2.4 & 0.0 \\
        \multicolumn{2}{@{}l}{CANDI \citep{pynadath2026candihybriddiscretecontinuousdiffusion}} & 69.0 & 36.1 & 4.0 & 0.0 \\
        \multicolumn{2}{@{}l}{FLM \citep{lee2026flowmaplanguagemodels}} & 69.1 & 21.5 & 2.8 & 0.0 \\
        \multicolumn{2}{@{}l}{S-FLM \citep{deschenaux2026language}} & 77.2 & 36.0 & 5.2 & 0.0 \\
        \multicolumn{2}{@{}l}{FMLM \citep{lee2026flowmaplanguagemodels}} & 75.0 & 42.2 & 4.9 & 0.0 \\
        \multicolumn{2}{@{}l}{FMLM+ \citep{agarwal2026posteriorrefinementfastlanguage}} & 67.4 & 29.9 & 4.1 & 0.3 \\
        \multicolumn{2}{@{}l}{DBTM \citep{tang2026discretebeckmanntransportmodels}} & -- & -- & -- & 0.8 \\
        \multicolumn{2}{@{}l}{AlphaDLM \citep{gushchin2026alphadlm}} & -- & -- & -- & 6.3 \\
        \midrule
        \multicolumn{2}{@{}l}{\textbf{\emph{a}FLM}} & 71.1 & 26.5 & 6.1 & 0.3 \\
        \midrule
        \multirow{4}{*}{\textbf{\ours{}}}
        & Sync. & 77.1 & 36.8 & 6.4 & 1.6 \\
        & EB & 79.1 & \textbf{47.8} & 11.8 & \textbf{11.4} \\
        & Top-$K$ & 81.4 & 46.4 & \textbf{14.4} & 10.5 \\
        & Alg. order & \textbf{82.5} & 44.5 & 10.0 & -- \\
        \bottomrule
    \end{tabular}
    \vspace{-1.4em}
\end{table}

\begin{figure}[t!]
    \centering
    \begin{minipage}[t]{0.49\linewidth}
        \centering
        \vspace{0pt}
        \includegraphics[width=\linewidth]{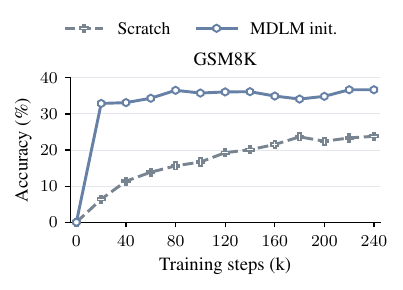}
        \vspace{-2.5em}
        \caption{GSM8K accuracy (\%) of the teacher \emph{a}FLM~(EB) with 32-step generation across training steps, with and without MDLM initialization.}
        \label{fig:mdlm_init_training}
    \end{minipage}\hfill
    \begin{minipage}[t]{0.49\linewidth}
        \centering
        \vspace{0pt}
        \includegraphics[width=\linewidth]{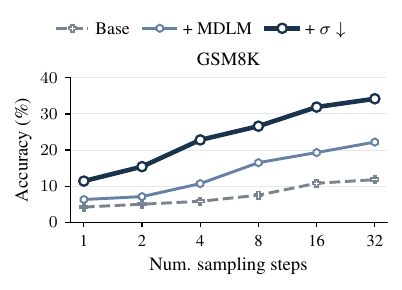}
        \vspace{-2.5em}
        \caption{GSM8K accuracy (\%) of \ours{}~(EB). ``+ MDLM'' initializes the teacher from MDLM, and ``+ $\sigma\!\downarrow$'' adds noise downscaling.}
        \label{fig:gsm8k_teacher_ablation}
    \end{minipage}
    \vspace{-1em}
\end{figure}

We distill \emph{a}FLM from \Cref{sec:exp_scheduling_rules} with EB, Top-$K$ and synchronous schedulers into \ours{}~(EB), \ours{}~(Top-$K$), and \ours{}~(Sync.), respectively, and test their one-step generation performance.

The results are shown in \Cref{tab:main_results_split}.
\ours{}~(EB) improves the one-step accuracy over its \emph{a}FLM teacher, 11.8\% vs. 6.1\% on Sudoku-Hard and 11.4\% vs. 0.3\% on GSM8K.
In contrast, \ours{}~(Sync.), which is distilled from the same teacher, reaches considerably lower performance, showing that the benefit of asynchronous schedulers carries over to one-step generation.
\ours{}~(Top-$K$) shows a similar gain and achieves the highest one-step accuracy on Sudoku-Hard, which shows that the benefit is not specific to a particular choice of the scheduler.
Through these advantages, \ours{} surpasses all one-step baselines on both Sudoku and GSM8K and establishes the new state of the art.

\begin{wraptable}[7]{r}{0.38\textwidth}
    \centering
    \vspace{-2em}
    \caption{Training cost (GPU hours) on Sudoku-Hard.}
    \label{tab:training_cost}
    \footnotesize
    \setlength{\tabcolsep}{3pt}
    \begin{tabular}{@{}lccc@{}}
        \toprule
        Schedule & Teacher & Distillation & Total \\
        \midrule
        Learned & 1.7h & 3.3h & 5.0h \\
        Prescribed & 0h & 3.2h & 3.2h \\
        \bottomrule
    \end{tabular}
\end{wraptable}

In addition, as in \Cref{prop:prescribed_loss_optimum}, \ours{} can be trained along a prescribed schedule from scratch without a pretrained teacher.
We show this with the algorithmic order on Sudoku.
As in \Cref{tab:main_results_split}, \ours{}~(Alg.~order) achieves 10.0\% one-step accuracy on Sudoku-Hard without a teacher, compared with 11.8\% of \ours{}~(EB).
As shown in \Cref{tab:training_cost}, this approach reduces the total training cost from 5 to 3.2 GPU hours by eliminating teacher pretraining.
Our method effectively distills both prescribed and learned schedules, and a useful prescribed schedule offers a more efficient route to one-step generation.

In \Cref{fig:mdlm_init_training,fig:gsm8k_teacher_ablation}, we analyze the role of our algorithmic choices (\Cref{sec:implementation}) of initializing the \emph{a}FLM teacher from MDLM and noise downscaling for distillation. Initializing the teacher from MDLM accelerates its training, where it reaches 32.9\% accuracy on GSM8K with 32-step generation after only 20k training steps, compared with 6.4\% from scratch. 
This suggests MDM initialization eases diffusion forcing training, as the combinations of tokenwise times close to zero or one resemble the problem solved by MDMs.
Distilling this stronger teacher improves the one-step accuracy of \ours{}~(EB) on GSM8K from 4.2\% to 6.3\%. Noise downscaling for distillation further improves \ours{} from 6.3\% to 11.4\% accuracy, confirming that it reduces the burden of fitting the flow map.

\begin{figure}[t]
    \centering
    \vspace{-2em}
    \includegraphics[width=0.95\textwidth]{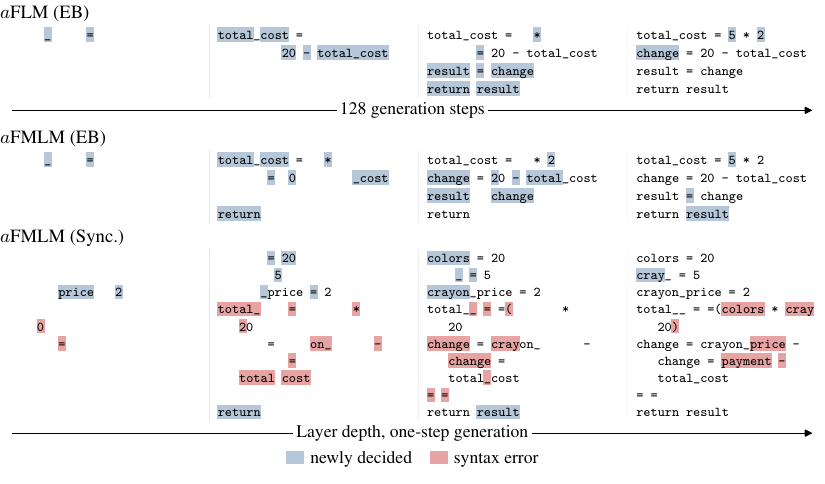}
    \vspace{-1.5em}
    \caption{A visualization of code solutions to a GSM8K problem generated by one-step \ours{}~(EB) (middle) and \ours{}~(Sync.) (bottom) in their layer-wise order, compared with the generation order of \emph{a}FLM~(EB) (top). Columns show when 10\%, 50\%, 85\%, and 100\% of the tokens are decided.}
    \vspace{-1em}
    \label{fig:exp_internalization}
\end{figure}

\vspace{-0.5em}
\subsection{Internalization of the schedule into \texorpdfstring{$a$}{a}FMLM (Q3)}
\label{sec:exp_internalization}

\begin{wraptable}[7]{r}{0.38\linewidth}
    \centering
    \vspace{-2em}
    \caption{Spearman correlation with the schedule-induced ordering.}
    \label{tab:exp_internalization}
    \footnotesize
    \setlength{\tabcolsep}{4pt}
    \begin{tabular}{@{}lcc@{}}
        \toprule
        Schedule & Sudoku-Hard & GSM8K \\
        \midrule
        Synchronous & 0.500 & 0.091 \\
        Asynchronous & \textbf{0.516} & \textbf{0.271} \\
        \bottomrule
    \end{tabular}
\end{wraptable}
We analyze the hidden computation of one-step \ours{} using logit lens \citep{belrose2023eliciting}, which decodes the hidden state of each layer using the output head.
For each token, we record the earliest layer from which its decoded prediction matches the final output and remains so, and compare this layer-wise order with the order induced by the schedule.
For Sudoku, we compare with the algorithmic order, and for GSM8K, we compare with the order in which the \emph{a}FLM teacher under the EB scheduler decides the tokens.
In \Cref{tab:exp_internalization}, we measure the alignment between the two orders by the Spearman correlation, computed per sample and averaged over the test set.
On both tasks, \ours{} shows a higher correlation with its schedule than \ours{}~(Sync.), which indicates that \ours{} internalizes the schedule into its hidden computation. This is consistent with the observation of \citet{asiimwe2026depth} in image generation that flow maps compress the flow across timesteps into their layers.
We further visualize an example in \Cref{fig:exp_internalization}.
\ours{}~(Sync.) does not follow the order of the teacher and produces non-parsable lines, showing that the order matters.
In contrast, \ours{}~(EB) resolves the tokens that the teacher generates early in its early layers and produces a parsable solution. 

%% file: sections/conclusion.tex
\section{Conclusion}
We introduced asynchronous flow map language models (\ours{}s), which distill the ordered generation of flow language models into one-step generators.
To this end, we modeled asynchronous generation with an asynchronous flow language model (\emph{a}FLM), which evolves a coupled flow of tokens and their tokenwise times. 
We then derived a self-distillation objective that learns the flow map of this flow, which jointly integrates both over any time interval.
Empirically, a well-designed scheduler substantially improves \emph{a}FLMs, and \ours{}s retain these advantage in one-step generation without evaluating the scheduler at test time. We also showed that \ours{}s internalize the order of the scheduler into their hidden computation. Promising future directions include jointly optimizing the scheduler with the denoiser and scaling \ours{}s to larger continuous language models.

\subsection*{Acknowledgments}
This work is supported by the National Research Foundation of Korea (RS-2024-00351212) and the Institute of Information \& communications Technology Planning \& Evaluation (IITP) (RS-2024-00509279, RS-2026-25518317, and RS-2026-25617720), funded by the Korean Government (MSIT). Satoshi Hayakawa was partially supported by JSPS KAKENHI (JP26K25548) and JST PRESTO (JPMJPR26K8).

%% file: sections/appendix_proofs.tex
\section{Proofs and Derivations}
Throughout, we assume that the velocity fields in this section are regular enough for their ODEs and continuity equations to have unique solutions on $[0,1)$, and that prescribed schedules are absolutely continuous.
We also assume $\omega^i_t>0$ whenever $0<t<1$ and $\tau^i_t<1$, and that the tokenwise times reach ${\bf 1}$ at $t=1$, which holds for the schedulers of \Cref{eq:common_schedule_rate} with $k^i_t>0$ for $0<t<1$ and bounded away from zero near $t=1$.
Quantities at $t=1$ are understood as limits $t\uparrow1$, and diagonal quantities at $t=s$ as limits $t\downarrow s$.

\subsection{Proof of Proposition \ref{prop:monotone}}
\label{app:monotone}
\monotone*

\begin{proof}
We parameterize each path by $r\in[0,1]$.

\paragraph{Scalar time.}
Let $t(r)$ be a continuous nondecreasing path with
\begin{equation*}
t(0)=0,\qquad t(1)=1.
\end{equation*}
Then $t(r)$ passes through every value in $[0,1]$ in order, so it is a reparameterization of $t(r)=r$.

\paragraph{Tokenwise times.}
It suffices to exhibit one infinite family of such paths that remain distinct up to reparameterization.

For $L\geq2$ and $a>0$, the path $\tau_a(r)=(r^a,r,\ldots,r)$ goes from $\mathbf0$ to $\mathbf1$ and satisfies $\tau_a^1=(\tau_a^2)^a$ along the way.

Since this relation does not change under reparameterization and differs across $a$, the paths $\tau_a$ are infinitely many and distinct up to reparameterization.
\end{proof}

\subsection{Proof of Proposition \ref{prop:exact_flow}}
\label{app:exact_flow}

\exactflow*

\begin{proof}
For each $({\bf x}_{\bf 0},{\bf x}_{\bf 1})$, let $\tau_t$ solve $\dot\tau_t=\omega_t(I_{\tau_t},\tau_t)$ with $\tau_0={\bf 0}$, and let $Z_t\coloneqq(I_{\tau_t},\tau_t)$.
Since position $i$ of the interpolant is $I^i_{\tau_t}=(1-\tau^i_t){\bf x}^i_{\bf 0}+\tau^i_t{\bf x}^i_{\bf 1}$, differentiating in $t$ gives
\begin{equation*}
\frac{\rm d}{{\rm d}t}I_{\tau_t}=\dot\tau_t({\bf x}_{\bf 1}-{\bf x}_{\bf 0}),
\qquad\text{so}\qquad
\dot Z_t=\bigl(\omega_t(Z_t)({\bf x}_{\bf 1}-{\bf x}_{\bf 0}),\ \omega_t(Z_t)\bigr).
\end{equation*}
Since $\omega_t(Z_t)$ is a function of $Z_t$, taking the conditional mean given $Z_t=({\bf x},\tau)$ and using \Cref{eq:exact_velocity} gives
\begin{equation*}
\begin{aligned}
\mathbb E[\dot Z_t\mid Z_t=({\bf x},\tau)]
&=\bigl(\omega_t({\bf x},\tau)\,\mathbb E[{\bf x}_{\bf 1}-{\bf x}_{\bf 0}\mid I_{\tau_t}={\bf x},\tau_t=\tau],\ \omega_t({\bf x},\tau)\bigr)\\
&=\bigl(\omega_t({\bf x},\tau)\,b^\omega_t({\bf x},\tau),\ \omega_t({\bf x},\tau)\bigr)
\eqqcolon u_t({\bf x},\tau).
\end{aligned}
\end{equation*}
Hence, for any smooth test function $\varphi$ with compact support, the law $q_t$ of $Z_t$ satisfies
\begin{equation*}
\frac{\rm d}{{\rm d}t}\mathbb E[\varphi(Z_t)]
=\mathbb E\bigl[\nabla\varphi(Z_t)\cdot\dot Z_t\bigr]
=\mathbb E\bigl[\nabla\varphi(Z_t)\cdot u_t(Z_t)\bigr],
\end{equation*}
so $q_t$ is a weak solution of the continuity equation
\begin{equation*}
\partial_t q_t+\nabla\cdot(q_t\,u_t)=0,\qquad q_0=p_{\bf 0}\otimes\delta_{\bf 0}.
\end{equation*}
The law of $({\bf x}_t,\tau_t)$ in the proposition solves the same equation with the same initial condition, since its velocity is $u_t$.
Since this solution is unique, the two laws coincide for every $t\in[0,1)$.
Finally, $\tau_t\to{\bf 1}$ as $t\uparrow1$, so $I_{\tau_t}\to{\bf x}_{\bf 1}$ for every $({\bf x}_{\bf 0},{\bf x}_{\bf 1})$.
Since ${\bf x}_t$ has the law of $I_{\tau_t}$, it converges in distribution to $p_{\bf 1}$.
\end{proof}

\subsection{Proof of Proposition \ref{prop:df_optimum}}
\label{app:df_optimum}
\dfoptimum*

\begin{proof}
Since ${\bf x}_{\bf 1}^i$ is the one-hot encoding of ${\bf y}^i$, the loss at position $i$ can be written as an inner product,
\begin{equation*}
\left[-\log \hat D_\tau^i(I_\tau)\right]_{{\bf y}^i}
=-\left\langle {\bf x}_{\bf 1}^i,\log \hat D_\tau^i(I_\tau)\right\rangle.
\end{equation*}
Conditioning on $(I_\tau,\tau)$ and using $\mathbb E[{\bf x}_{\bf 1}\mid I_\tau,\tau]=D_\tau(I_\tau)$, which holds because $\tau$ is independent of $({\bf x}_{\bf 0},{\bf x}_{\bf 1})$, its expectation is
\begin{equation*}
\mathbb E\left[-\left\langle {\bf x}_{\bf 1}^i,\log \hat D_\tau^i(I_\tau)\right\rangle\right]
=\mathbb E\left[-\left\langle \mathbb E[{\bf x}_{\bf 1}^i\mid I_\tau, \tau],\log \hat D_\tau^i(I_\tau)\right\rangle\right]
=\mathbb E\left[-\left\langle D_\tau^i(I_\tau),\log \hat D_\tau^i(I_\tau)\right\rangle\right].
\end{equation*}
Adding and subtracting $\langle D_\tau^i(I_\tau),\log D_\tau^i(I_\tau)\rangle$ and summing over positions gives
\begin{equation}
\mathcal L_\mathsf{CE}(\hat D)
=\mathbb E\left[\sum_{i=1}^L H(D_\tau^i(I_\tau))\right]
+\mathbb E\left[\sum_{i=1}^L\mathsf{KL}\!\left(
D_\tau^i(I_\tau)\,\|\,\hat D_\tau^i(I_\tau)\right)\right],
\label{eq:async_ce_decomposition}
\end{equation}
where $H(p)=-\langle p,\log p\rangle$ is the entropy.
The first term does not depend on $\hat D$, and the second term is nonnegative and vanishes exactly when $\hat D_\tau^i(I_\tau)=D_\tau^i(I_\tau)$ for every $i$.
Thus, $\hat D^\star$ minimizes $\mathcal L_\mathsf{CE}$ if and only if
\begin{equation}
\hat D_\tau^\star(I_\tau)=D_\tau(I_\tau)
\label{eq:async_ce_minimizer}
\end{equation}
almost surely under the training distribution.
\end{proof}

\subsection{Support of the diffusion forcing training distribution}
\label{app:df_support}
We show that the training distribution of $(I_\tau,\tau)$ in \Cref{eq:df_loss} has support $\mathbb R^{L\times|V|}\times[0,1]^L$, so that every state $({\bf x}_t,\tau_t)$ of an asynchronous flow lies in it regardless of the scheduler.

\begin{proof}
Fix $\tau\in[0,1)^L$ and ${\bf x}_{\bf 1}$.
Since ${\bf x}_{\bf 0}\sim\mathsf N(0,I)$ is independent of ${\bf x}_{\bf 1}$ and $\tau$, each position of the interpolant is distributed as
\begin{equation*}
I^i_\tau=(1-\tau^i){\bf x}^i_{\bf 0}+\tau^i{\bf x}^i_{\bf 1}
\sim\mathsf N\big(\tau^i{\bf x}^i_{\bf 1},\,(1-\tau^i)^2I\big),
\end{equation*}
independently across positions.
Since $1-\tau^i>0$, $I_\tau$ has a positive density on $\mathbb R^{L\times|V|}$.
Moreover, each $\tau^i\sim\mathsf U[0,1]$ independently, so $\tau$ has a positive density on $[0,1)^L$.
The joint distribution of $(I_\tau,\tau)$ therefore has a positive density on $\mathbb R^{L\times|V|}\times[0,1)^L$, and its support is the closure of this set,
\begin{equation*}
\overline{\mathbb R^{L\times|V|}\times[0,1)^L}=\mathbb R^{L\times|V|}\times[0,1]^L,
\end{equation*}
which contains every state $({\bf x}_t,\tau_t)$ of an asynchronous flow.
\end{proof}

\subsection{Proof of Corollary \ref{prop:async_flow}}
\label{app:async_flow}

\asyncflow*

\begin{proof}
Since $\omega$ does not depend on ${\bf x}$, every interpolant follows the same tokenwise times $\tau_t=\tau(t)$ with $\omega_t=\dot\tau(t)$.
The condition $\tau_t=\tau(t)$ in \Cref{eq:exact_velocity} therefore holds for every $({\bf x}_{\bf 0},{\bf x}_{\bf 1})$, and 
\begin{equation*}
    b^\omega_t({\bf x},\tau(t))=\mathbb E[{\bf x}_{\bf 1}-{\bf x}_{\bf 0}\mid I_{\tau(t)}={\bf x}]=b_{\tau(t)}({\bf x}).
\end{equation*}
Substituting into \Cref{eq:state_dependent_dynamics} gives \Cref{eq:async_flow}, whose velocity $\omega_t b_{\tau(t)}=\omega_t b^\omega_t$ coincides with that of \Cref{prop:exact_flow}.
Hence, ${\bf x}_t$ has the law of $I_{\tau(t)}$, which is $p_{\tau(t)}$ for every $t\in[0,1)$, and ${\bf x}_t$ converges in distribution to $p_{\bf 1}$ as $t\uparrow1$.
\end{proof}

\subsection{Proof of Proposition \ref{prop:state_async_flow_map}}
\label{app:state_async_flow_map}

\stateasyncflowmap*

\begin{proof}
Let $({\bf x}_r,\tau_r)= F_{s,r}({\bf z})$, where ${\bf z}=({\bf x},\tau)$. Integrating the coupled dynamics \Cref{eq:state_dependent_dynamics} from $s$ to $t$ gives the first identity.

For $s\leq u\leq t<1$, both $r\mapsto F_{s,r}({\bf z})$ and $r\mapsto F_{u,r}( F_{s,u}({\bf z}))$ solve the coupled dynamics on $[u,t]$ with the same state $F_{s,u}({\bf z})$ at $r=u$. Uniqueness of solutions therefore gives $F_{s,t}({\bf z})=F_{u,t}(F_{s,u}({\bf z}))$.
\end{proof}

\subsection{Proof of Proposition \ref{prop:state_two_time_denoiser}}
\label{app:state_two_time_denoiser}

\statetwotimedenoiser*

\begin{proof}
Throughout, let ${\bf z}=({\bf x},\tau)$ and $({\bf x}_r,\tau_r)= F_{s,r}({\bf z})$ for $r\geq s$, so that $({\bf x}_s,\tau_s)=({\bf x},\tau)$.
By \Cref{eq:joint_average_velocities}, the averages $v_{s,t}$ and $\eta_{s,t}$ satisfy
\begin{align}
(t-s)\,v_{s,t}({\bf z})
&=\int_s^t \omega_r({\bf x}_r,\tau_r)\, b_{\tau_r}({\bf x}_r)\,{\rm d}r
={\bf x}_t-{\bf x},\nonumber\\
(t-s)\,\eta_{s,t}({\bf z})
&=\int_s^t \omega_r({\bf x}_r,\tau_r)\,{\rm d}r
=\tau_t-\tau.
\label{eq:state_displacements}
\end{align}
Substituting \Cref{eq:state_displacements} into \Cref{eq:state_two_time_denoiser},
\begin{equation}
\delta_{s,t}^i({\bf z})
={\bf x}^i+\frac{1-\tau^i}{\eta_{s,t}^i({\bf z})}v_{s,t}^i({\bf z})
={\bf x}^i+\frac{1-\tau^i}{\tau^i_t-\tau^i}
\bigl({\bf x}_t^i-{\bf x}^i\bigr).
\label{eq:state_proof_definition}
\end{equation}
Solving \Cref{eq:state_proof_definition} for ${\bf x}_t$ and \Cref{eq:state_displacements} for $\tau_t$ proves (i).

Next, we write $\delta_{s,t}$ as a weighted average of the denoiser along the trajectory. Recall the coupled dynamics,
\begin{equation*}
\dot\tau_r^i=\omega_r^i({\bf x}_r,\tau_r),\qquad
\dot{\bf x}_r^i=\frac{\omega_r^i({\bf x}_r,\tau_r)}{1-\tau_r^i}
\bigl(D_{\tau_r}^i({\bf x}_r)-{\bf x}_r^i\bigr).
\end{equation*}
The product rule gives
\begin{align}
\frac{\rm d}{{\rm d}r}\frac{{\bf x}_r^i}{1-\tau_r^i}
&=\frac{\dot{\bf x}_r^i}{1-\tau_r^i}
+\frac{\omega_r^i({\bf x}_r,\tau_r){\bf x}_r^i}{(1-\tau_r^i)^2}\nonumber\\
&=\frac{\omega_r^i({\bf x}_r,\tau_r)}{(1-\tau_r^i)^2}D_{\tau_r}^i({\bf x}_r).
\label{eq:state_integrating_factor}
\end{align}
Integrating \Cref{eq:state_integrating_factor} and using the tokenwise time dynamics,
\begin{align}
\frac{{\bf x}_t^i}{1-\tau_t^i}-\frac{{\bf x}^i}{1-\tau^i}
&=\int_s^t\frac{\omega_r^i({\bf x}_r,\tau_r)}{(1-\tau_r^i)^2}D_{\tau_r}^i({\bf x}_r)\,{\rm d}r,
\label{eq:state_integrated_velocity}\\
\int_s^t\frac{\omega_r^i({\bf x}_r,\tau_r)}{(1-\tau_r^i)^2}\,{\rm d}r
&=\frac1{1-\tau_t^i}-\frac1{1-\tau^i}
=\frac{\tau_t^i-\tau^i}{(1-\tau^i)(1-\tau_t^i)}.
\label{eq:state_weight_normalization}
\end{align}
Substituting \Cref{eq:state_integrated_velocity,eq:state_weight_normalization} into \Cref{eq:state_proof_definition},
\begin{equation}
\delta_{s,t}^i({\bf z})
=\frac{\displaystyle\int_s^t\frac{\omega_r^i({\bf x}_r,\tau_r)}{(1-\tau_r^i)^2}D_{\tau_r}^i({\bf x}_r)\,{\rm d}r}
{\displaystyle\int_s^t\frac{\omega_r^i({\bf x}_r,\tau_r)}{(1-\tau_r^i)^2}\,{\rm d}r}
\in\Delta^{|V|-1}.
\label{eq:state_weighted_denoiser}
\end{equation}
For (ii), substituting the diagonal of \Cref{eq:joint_average_velocities} into \Cref{eq:state_two_time_denoiser} gives
\begin{equation*}
\delta_{t,t}({\bf z})={\bf x}+\frac{\mathbf1-\tau}{\omega_t({\bf x},\tau)}\,\omega_t({\bf x},\tau)\,b_\tau({\bf x})=D_\tau({\bf x}),
\qquad
\eta_{t,t}({\bf z})=\omega_t({\bf x},\tau).
\end{equation*}
These diagonal values are also the limits of the off-diagonal ones. Since the weights in \Cref{eq:state_weighted_denoiser} are nonnegative,
\begin{equation*}
\|\delta_{s,t}^i({\bf z})-D_\tau^i({\bf x})\|
\leq\sup_{r\in[s,t]}\|D_{\tau_r}^i({\bf x}_r)-D_\tau^i({\bf x})\|
\xrightarrow[t\downarrow s]{}0,
\end{equation*}
and \Cref{eq:state_displacements} gives
\begin{equation*}
\lim_{t\downarrow s}\eta_{s,t}({\bf z})
=\lim_{t\downarrow s}\frac1{t-s}\int_s^t \omega_r({\bf x}_r,\tau_r)\,{\rm d}r
=\omega_s({\bf x},\tau).
\end{equation*}

For (iii), fix $s<u<t$ and write ${\bf z}_u= F_{s,u}({\bf z})$.
By \Cref{eq:state_displacements}, the tokenwise time components of the direct map $F_{s,t}({\bf z})$ and the composed map $F_{u,t}({\bf z}_u)$ are
\begin{equation*}
\tau+(t-s)\,\eta_{s,t}({\bf z})
\qquad\text{and}\qquad
\tau+(u-s)\,\eta_{s,u}({\bf z})+(t-u)\,\eta_{u,t}({\bf z}_u),
\end{equation*}
which agree if and only if the identity for $\eta$ in \Cref{eq:state_denoiser_composition} holds.
When this identity holds, let $\tau_t$ denote the common tokenwise times, and write
\begin{equation*}
a_i=\tau^i,\qquad b_i=\tau_u^i,\qquad c_i=\tau_t^i.
\end{equation*}
Applying (i) to $[s,u]$ and then to $[u,t]$ gives the token state component of the composed map,
\begin{align}
&\frac{1-c_i}{1-b_i}
\left[\frac{1-b_i}{1-a_i}{\bf x}^i+\frac{b_i-a_i}{1-a_i}\delta_{s,u}^i({\bf z})\right]
+\frac{c_i-b_i}{1-b_i}\delta_{u,t}^i({\bf z}_u)\nonumber\\
&=\frac{1-c_i}{1-a_i}{\bf x}^i
+\frac{c_i-a_i}{1-a_i}
\bigl[\gamma_i\delta_{s,u}^i({\bf z})+(1-\gamma_i)\delta_{u,t}^i({\bf z}_u)\bigr],
\label{eq:state_composed_transition}
\end{align}
where
\begin{equation*}
\gamma_i=\frac{(1-c_i)(b_i-a_i)}{(1-b_i)(c_i-a_i)},\qquad
1-\gamma_i=\frac{(1-a_i)(c_i-b_i)}{(1-b_i)(c_i-a_i)}.
\end{equation*}
Both are nonnegative since $a_i<b_i<c_i<1$, so $\gamma_i\in[0,1]$.
By (i), the token state component of the direct map is $\frac{1-c_i}{1-a_i}{\bf x}^i+\frac{c_i-a_i}{1-a_i}\delta_{s,t}^i({\bf z})$, so subtracting it from \Cref{eq:state_composed_transition} leaves
\begin{equation*}
\frac{c_i-a_i}{1-a_i}
\bigl[\gamma_i\delta_{s,u}^i({\bf z})+(1-\gamma_i)\delta_{u,t}^i({\bf z}_u)-\delta_{s,t}^i({\bf z})\bigr].
\end{equation*}
Since $c_i>a_i$, the two token state components then agree if and only if the identity for $\delta$ in \Cref{eq:state_denoiser_composition} holds.
Hence the semigroup condition \Cref{eq:state_semigroup_condition} holds if and only if both identities hold, which proves (iii).
\end{proof}

\subsection{Proof of Proposition \ref{prop:state_loss_optimum}}
\label{app:state_loss_optimum}
\stateLossOptimum*

\begin{proof}
Setting $\hat\delta=\delta$ and $\hat\eta=\eta$, the induced map $\hat F$ is the exact flow map $F$ by \Cref{prop:state_two_time_denoiser} (i), and \Cref{prop:state_two_time_denoiser} (ii) gives
\begin{equation*}
\delta_{t,t}({\bf z})=D_\tau({\bf x}),\qquad
\eta_{t,t}({\bf z})=\omega_t({\bf x},\tau),
\end{equation*}
so both diagonal terms vanish.
Moreover, \Cref{prop:state_two_time_denoiser} (iii) gives
\begin{equation*}
\delta_{s,t}({\bf z})=\mathcal T^\delta_{s,u,t}({\bf z}),\qquad
\eta_{s,t}({\bf z})=\mathcal T^\eta_{s,u,t}({\bf z}),
\end{equation*}
so both off-diagonal terms vanish as well, and $\mathcal L(\delta,\eta)=0$.

Conversely, all four terms are nonnegative and $\lambda>0$, so $\mathcal L(\hat\delta,\hat\eta)=0$ if and only if each term vanishes almost surely under the training distribution.
Since KL divergence and the squared norm vanish exactly when their arguments agree, and stop-gradient leaves the target values unchanged, this gives
\begin{equation*}
\hat\delta_{t,t}({\bf z})=D_\tau({\bf x}),\qquad
\hat\eta_{t,t}({\bf z})=\omega_t({\bf x},\tau),\qquad
\hat\delta_{s,t}({\bf z})=\mathcal T^\delta_{s,u,t}({\bf z}),\qquad
\hat\eta_{s,t}({\bf z})=\mathcal T^\eta_{s,u,t}({\bf z}).
\end{equation*}
The first two are \Cref{eq:state_denoiser_diagonal}, and expanding the targets in the last two gives \Cref{eq:state_denoiser_composition}, which hold for $\hat\delta$ and $\hat\eta$.
\end{proof}

\subsection{Proof of Corollary \ref{prop:prescribed_loss_optimum}}
\label{app:prescribed_distillation}
\prescribedLossOptimum*

\begin{proof}
We first show that the tokenwise time terms of \Cref{eq:joint_self_distillation_loss} vanish.
By assumption, $\hat\eta$ is fixed to the exact $\eta$, and since $\tau_t=\tau(t)$ for every state, \Cref{eq:joint_average_velocities} gives
\begin{equation*}
\eta_{s,t}=\frac{\tau(t)-\tau(s)}{t-s},\qquad
\eta_{t,t}=\dot\tau(t)=\omega_t.
\end{equation*}
The diagonal tokenwise time term therefore vanishes.
Substituting $\eta$ into \Cref{eq:state_composition_targets} gives
\begin{equation*}
\mathcal T^\eta_{s,u,t}
=\frac{(u-s)\,\eta_{s,u}+(t-u)\,\eta_{u,t}}{t-s}
=\frac{(\tau(u)-\tau(s))+(\tau(t)-\tau(u))}{t-s}
=\eta_{s,t},
\end{equation*}
so the off-diagonal tokenwise time term vanishes as well.

We next rewrite the token state terms.
Since ${\bf x}_t\sim p_{\tau(t)}$ by \Cref{prop:async_flow}, the states in \Cref{eq:joint_self_distillation_loss} are ${\bf z}=(I_{\tau(t)},\tau(t))$, and $\hat F_{s,u}$ and $\gamma$ in $\mathcal T^\delta_{s,u,t}$ are evaluated at $\tau(s)$, $\tau(u)$, and $\tau(t)$.
The off-diagonal token state term is thus the first term of \Cref{eq:prescribed_self_distillation_loss}.
For the diagonal term, ${\bf x}_{\bf 1}^i$ is one-hot and has zero entropy, so
\begin{equation*}
\mathsf{KL}({\bf x}_{\bf 1}\,\|\,\hat\delta_{t,t}({\bf z}))
=\sum_{i=1}^L\left[-\log\hat\delta^i_{t,t}({\bf z})\right]_{{\bf y}^i},
\end{equation*}
which is the tokenwise cross-entropy of \Cref{eq:df_loss} with $\hat D_{\tau(t)}=\hat\delta_{t,t}$.
Applying the decomposition \Cref{eq:async_ce_decomposition} with $\tau=\tau(t)$ gives
\begin{equation*}
\begin{aligned}
\mathbb E_{t,{\bf z},{\bf x}_{\bf 1}}\,\big[\mathsf{KL}({\bf x}_{\bf 1}\,\|\,\hat\delta_{t,t}({\bf z}))\big]
&=C+\mathbb E_{t,{\bf z}}\,\big[\mathsf{KL}(D_{\tau(t)}(I_{\tau(t)})\,\|\,\hat\delta_{t,t}({\bf z}))\big],\\
C&\coloneqq\mathbb E_{t,{\bf z}}\left[\sum_{i=1}^L H(D_{\tau(t)}^{i}(I_{\tau(t)}))\right],
\end{aligned}
\end{equation*}
where KL sums over token positions and $C$ does not depend on $\hat\delta$.
Hence \Cref{eq:prescribed_self_distillation_loss} equals \Cref{eq:joint_self_distillation_loss} plus $C$.
Since \Cref{eq:joint_self_distillation_loss} is nonnegative and vanishes at the exact $\delta$ by \Cref{prop:state_loss_optimum}, the minimum of \Cref{eq:prescribed_self_distillation_loss} is $C$, and its minimizers are exactly the zero-loss solutions of \Cref{prop:state_loss_optimum}.
\end{proof}

%% file: sections/appendix_schedulers.tex
\section{Scheduler Parameterization}
\subsection{Schedule construction}
\label{app:schedule_construction}
\begin{figure}[t]
    \centering
    \includegraphics[width=0.8\linewidth]{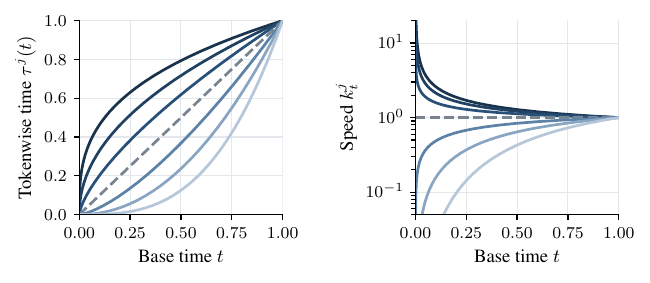}
    \caption{Power schedules $\tau^j(t)=t^{\alpha_j}$ with $\beta=\log 3$ (left) and their speeds $k_t^j$ (right). Darker curves denote higher-priority positions, with exponents ranging from $\alpha_j=1/3$ (darkest) to $\alpha_j=3$ (lightest), and the dashed line denotes the synchronous schedule.}
    \label{fig:power_schedule}
\end{figure}
All schedulers in our experiments take the common form of \Cref{eq:common_schedule_rate},
\begin{equation*}
    \omega_t^i({\bf x},\tau)=\frac{1-\tau^i}{1-t}\,k_t^i({\bf x},\tau), \qquad t<1.
\end{equation*}
This form sends $\tau_t\to\mathbf1$ as $t\uparrow1$ whenever $k^i$ is bounded away from zero near $t=1$.
Writing $k^i_r$ for $k^i_r({\bf x}_r,\tau_r)$ along the trajectory, each tokenwise time satisfies
\begin{equation*}
    \frac{\rm d}{{\rm d}r}\log(1-\tau^i_r)
    =-\frac{\dot\tau^i_r}{1-\tau^i_r}
    =-\frac{k^i_r}{1-r},
\end{equation*}
and integrating over $[s,t]$ gives
\begin{equation}
    1-\tau^i_t
    =(1-\tau^i_s)\exp\!\left(-\int_s^t\frac{k^i_r}{1-r}\,{\rm d}r\right).
    \label{eq:tau_gap}
\end{equation}
Since $k^i_r\geq0$ and $\tau^i_0=0$, the tokenwise time $\tau^i_t$ never decreases and stays in $[0,1)$ for every $t<1$.
Moreover, if $k^i_r\geq c$ for some $c>0$ and all $r\in[s,1)$, then $\int_s^t{\rm d}r/(1-r)=\log\frac{1-s}{1-t}$ gives
\begin{equation*}
    1-\tau^i_t\leq(1-\tau^i_s)\left(\frac{1-t}{1-s}\right)^{c}\to0
    \quad\text{as } t\uparrow1.
\end{equation*}
All schedulers in our experiments satisfy this condition, since the top-$K$ confidence and entropy-bounded schedulers use $k^i\geq k_{\min}>0$, and the power schedules have $k^j_t\to1$ as $t\uparrow1$.

The schedulers differ in what $k^i$ depends on.
Let
\begin{equation}
    p_i=D_\tau^i({\bf x}),\qquad
    c_i=\|p_i\|_\infty,\qquad h_i=H(p_i)
\end{equation}
denote the predicted clean-token distribution, its confidence, and its entropy, respectively. All position indices and statistics below refer to generated tokens, and fixed context tokens remain unchanged.

\paragraph{Power schedules.}\label{app:power_schedule_details}
Several of the schedulers below use the power schedule
\begin{equation}
    \tau^j(t)=t^{\alpha_j}
    \label{eq:power_schedule}
\end{equation}
(\Cref{fig:power_schedule}). Its rate $\omega_t^j=\alpha_j t^{\alpha_j-1}$ takes the common form, with $k$ depending only on $t$,
\begin{equation}
    k_t^j=\omega_t^j\,\frac{1-t}{1-\tau^j(t)}=\alpha_j t^{\alpha_j-1}\,\frac{1-t}{1-t^{\alpha_j}},
\end{equation}
which tends to one as $t\uparrow1$.
Given an order over the generated positions, let $\rho_j\in[0,1]$ denote the normalized rank of position $j$, with zero denoting the highest priority, and set
\begin{equation}
    \alpha_j=\exp\bigl[\beta\,(2\rho_j-1)\bigr],
    \label{eq:power_exponents}
\end{equation}
so that the exponents range from $e^{-\beta}$ to $e^{\beta}$ and positions with higher priority are denoised faster.
Except for the synchronous schedule, we use $\beta=\log 3$, so that the exponents range from $1/3$ to $3$. We use the following power schedules.
\begin{itemize}[leftmargin=*,itemsep=2pt]
    \item \textbf{Synchronous.} Setting $\beta=0$ gives $k^j=1$ and $\tau^j(t)=t$, which recovers the synchronous schedule of \citet{lee2026flowmaplanguagemodels}.
    \item \textbf{Left-to-right.} $\rho_j$ is the normalized position of token $j$ among the generated positions.
    \item \textbf{Algorithmic order for Sudoku.} $\rho_j=r_{\bf c}(j)/R_{\bf c}$, where $r_{\bf c}(j)$ is the round in which cell $j$ is resolved and $R_{\bf c}$ is the last round (\Cref{app:sudoku_order}). Cells resolved in the same round share the same exponent.
    \item \textbf{Initial entropy.} Following \citet{bansal2026token}, $\rho_j$ is the rank of position $j$ by increasing entropy of the initial prediction $D_{\bf 0}({\bf x}_0)$, kept fixed during generation, so that positions with lower entropy are denoised earlier. $k$ thus depends on the token states only through the initial state ${\bf x}_0$.
\end{itemize}

The following two schedulers instead set $k^i$ from the current predictions at every step.

\paragraph{Top-$K$ confidence scheduler.}
Following MaskGIT \citep{chang2022maskgit}, we advance the $K$ most confident positions faster, where $K$ is a fraction $f(t)=1-\cos(\pi t/2)$ of the generated positions and thus increases over time.
Let $\widetilde c_i$ denote the confidence $c_i$ standardized over the generated positions, and let $q(t)$ be its quantile at level $1-f(t)$, so that a fraction $f(t)$ of positions have $\widetilde c_i>q(t)$.
We set
\begin{equation}
    k^i=k_{\min}+A\,\mathrm{sigmoid}\!\left(\frac{\widetilde c_i-q(t)}{T}\right),
    \label{eq:policy_confidence}
\end{equation}
where $k_{\min}>0$ is the speed of the remaining positions, $A>0$ is the additional speed for the selected ones, and $T>0$ is a temperature that controls how sharply $k$ changes around the threshold.

\paragraph{Entropy-bounded scheduler.}
Following EB-Sampler \citep{benhamu2025accelerated},
we order token positions by decreasing confidence or
increasing entropy.
Let $\pi_j$ denote the $j$-th position in this order.
For the first $j$ positions, we sum their entropies
and subtract the largest:
\begin{equation}
    B_j=\sum_{\ell=1}^{j}h_{\pi_\ell}
        -\max_{1\leq\ell\leq j}h_{\pi_\ell}.
\end{equation}
We assign larger $k$ to positions before
$B_j$ exceeds a threshold $\epsilon$, using a sigmoid
to decrease $k$ gradually:
\begin{equation}
    k^{\pi_j}=k_{\min}+A\,
    \mathrm{sigmoid}\!\left(\frac{\epsilon-B_j}{T}\right).
    \label{eq:policy_soft_eb}
\end{equation}
Here $A>0$ controls the additional speed, and $T>0$ controls how sharply it decreases around $\epsilon\geq0$.

\subsection{Tokenwise time rate reparameterization}
\label{app:effective_tokenwise_time_rate}

Taking the logarithm of \Cref{eq:tau_gap} at each position gives
\begin{equation*}
\log\frac{\mathbf1-\tau_s}{\mathbf1-\tau_t}=\int_s^t\frac{k_r({\bf x}_r,\tau_r)}{1-r}\,{\rm d}r,
\qquad
\log\frac{1-s}{1-t}=\int_s^t\frac{1}{1-r}\,{\rm d}r.
\end{equation*}
Dividing the first by the second shows that $\kappa_{s,t}$ defined in \Cref{sec:implementation} is the average of $k_r$ along the trajectory weighted by $1/(1-r)$,
\begin{equation*}
\kappa_{s,t}=\frac{\int_s^t k_r({\bf x}_r,\tau_r)/(1-r)\,{\rm d}r}{\int_s^t 1/(1-r)\,{\rm d}r}.
\end{equation*}
We now derive the relation between $\hat\kappa_{s,s}$ and the diagonal rate $\hat\eta_{s,s}$.
By \Cref{prop:state_two_time_denoiser}(i), the left side of \Cref{eq:joint_clock_parameterization} is the tokenwise time component $\tau_s+(t-s)\,\hat\eta_{s,t}({\bf z}_s)$ of the flow map.
Its derivative in $t$ at $t=s$ is therefore $\hat\eta_{s,s}({\bf z}_s)$.

Writing the power in \Cref{eq:joint_clock_parameterization} as an exponential,
\begin{equation*}
\left(\frac{1-t}{1-s}\right)^{\hat\kappa_{s,t}}
=\exp\!\left[\hat\kappa_{s,t}\log\frac{1-t}{1-s}\right],
\end{equation*}
and differentiating with respect to $t$ by the chain and product rules gives
\begin{equation*}
\partial_t\left(\frac{1-t}{1-s}\right)^{\hat\kappa_{s,t}}
=\left(\frac{1-t}{1-s}\right)^{\hat\kappa_{s,t}}
\odot\left[\partial_t\hat\kappa_{s,t}\log\frac{1-t}{1-s}-\frac{\hat\kappa_{s,t}}{1-t}\right].
\end{equation*}
At $t=s$, we have $\frac{1-t}{1-s}=1$ and $\log\frac{1-t}{1-s}=0$, so the term involving $\partial_t\hat\kappa_{s,t}$ vanishes and
\begin{equation*}
\left.\partial_t\left(\frac{1-t}{1-s}\right)^{\hat\kappa_{s,t}}\right|_{t=s}
=-\frac{\hat\kappa_{s,s}}{1-s}.
\end{equation*}
Differentiating both sides of \Cref{eq:joint_clock_parameterization} at $t=s$ therefore gives
\begin{equation*}
\hat\eta_{s,s}({\bf z}_s)
=-(\mathbf1-\tau_s)\odot\left(-\frac{\hat\kappa_{s,s}({\bf z}_s)}{1-s}\right)
=\frac{\mathbf1-\tau_s}{1-s}\odot\hat\kappa_{s,s}({\bf z}_s).
\end{equation*}
On the other hand, the scheduler of \Cref{eq:common_schedule_rate} at the same state is
\begin{equation*}
\omega_s({\bf x}_s,\tau_s)=\frac{\mathbf1-\tau_s}{1-s}\odot k_s({\bf x}_s,\tau_s).
\end{equation*}
Subtracting the two expressions gives
\begin{equation}
\hat\eta_{s,s}({\bf z}_s)-\omega_s({\bf x}_s,\tau_s)
=\frac{\mathbf1-\tau_s}{1-s}\odot\bigl(\hat\kappa_{s,s}({\bf z}_s)-k_s({\bf x}_s,\tau_s)\bigr).
\label{eq:effective_clock_tangent}
\end{equation}
Since every entry of $(\mathbf1-\tau_s)/(1-s)$ is positive for $\tau_s<\mathbf1$ and $s<1$, the diagonal identity $\hat\eta_{s,s}=\omega_s$ holds if and only if $\hat\kappa_{s,s}=k_s$.
Moreover, taking the squared norm of \Cref{eq:effective_clock_tangent} gives
\begin{equation*}
\|\hat\eta_{s,s}({\bf z}_s)-\omega_s({\bf x}_s,\tau_s)\|_2^2
=\sum_{i=1}^L\left(\frac{1-\tau_s^i}{1-s}\right)^2\bigl(\hat\kappa^i_{s,s}({\bf z}_s)-k^i_s({\bf x}_s,\tau_s)\bigr)^2,
\end{equation*}
so the implemented term $\lambda\|k_s-\hat\kappa_{s,s}\|_2^2$ only removes these positive weights, and both terms vanish exactly when $\hat\kappa_{s,s}=k_s$.

%% file: sections/appendix_experiments.tex
\section{Experiment Details}
\label{app:exp_details}

\begin{table}[t]
    \centering
    \caption{\textbf{Per-task hyperparameters.} All models on the same task share the architecture and optimizer, and differ only in the training and inference schedules.}
    \label{tab:hyperparameters}
    \small
    \begin{tabular}{@{}lcc@{}}
        \toprule
        & TinyGSM / GSM8K & Sudoku \\
        \midrule
        \multicolumn{3}{@{}l}{\textit{Architecture (DiT)}} \\
        Hidden size & 768 & 512 \\
        adaLN cond.\ dim & 128 & 128 \\
        Layers & 12 & 8 \\
        Attention heads & 12 & 8 \\
        FFN multiplier & 4 & 4 \\
        Dropout & 0.1 & 0.1 \\
        Sequence length $L$ & 512 & 180 \\
        Vocabulary $|V|$ & 49152 (SmolLM-135M) & 12 \\
        Parameters (approx.) & 168M & 29M \\
        \midrule
        \multicolumn{3}{@{}l}{\textit{Teacher training}} \\
        Optimizer & \multicolumn{2}{c}{AdamW ($\beta_1{=}0.9$, $\beta_2{=}0.999$, $\epsilon{=}10^{-8}$)} \\
        Weight decay & \multicolumn{2}{c}{0} \\
        Gradient clipping & \multicolumn{2}{c}{1.0} \\
        Schedule & \multicolumn{2}{c}{Constant w/ linear warmup} \\
        Warmup steps & \multicolumn{2}{c}{2\,500} \\
        Peak LR & $3\times10^{-4}$ & $3\times10^{-4}$ \\
        Total steps & 80,000 & 10,000 \\
        Global batch size & 512 & 256 \\
        Precision & \multicolumn{2}{c}{bfloat16 mixed} \\
        EMA decay & \multicolumn{2}{c}{0.9999} \\
        \midrule
        \multicolumn{3}{@{}l}{\textit{Distillation}} \\
        Total steps & 170,000 & 10,000 \\
        Peak LR & $3\times10^{-4}$ & $3\times10^{-4}$ \\
        Global batch size & 512 & 256 \\
        Training steps per curriculum round & 2,000 & 2,000 \\
        Initial rollout steps & 32 & 32 \\
        Noise scale $\sigma$ & 0.1 & 1.0 \\
        \bottomrule
    \end{tabular}
\end{table}

\subsection{Training details}
\label{app:hyperparameters}
\paragraph{Architecture and hyperparameters.} 
All models on the same task share the architecture, and the tokenwise time $\tau\in[0,1]^L$ is embedded per position for the \emph{a}FLM teacher and \ours{}.
Per-task hyperparameters are listed in \Cref{tab:hyperparameters}.
We train all models with AdamW~\citep{loshchilov2017decoupled} with $\beta_1=0.9$, $\beta_2=0.999$, $\epsilon=10^{-8}$, and no weight decay.
The learning rate is linearly warmed up to $3\times10^{-4}$ over 2,500 steps and kept constant afterwards, and gradients are clipped to norm 1.0.
We use bfloat16 mixed precision and keep an exponential moving average of the weights with decay 0.9999, which is used for evaluation. All models are trained on NVIDIA A100 and H200 GPUs.

\paragraph{\emph{a}FLM training.}
For each task, we train one \emph{a}FLM teacher with the diffusion forcing loss of \Cref{eq:df_loss} for the number of steps in \Cref{tab:hyperparameters}.
The teacher is initialized from MDLM (\Cref{sec:implementation}) unless stated otherwise. We use the \emph{a}FLM checkpoint at 80k steps, and \Cref{fig:mdlm_init_training} continues training up to 240k steps only to compare the two initializations.
For Alg.\ order$^\dagger$ in \Cref{fig:scheduled_inference}, we instead train the denoiser only along the algorithmic order of \Cref{app:sudoku_order}.

\paragraph{Distillation.}
\ours{} (EB), \ours{} (Top-$K$), and \ours{} (Sync.) are initialized from the \emph{a}FLM teacher and distilled with \Cref{eq:joint_self_distillation_loss}, while \ours{} (Alg.\ order) is trained from scratch with \Cref{eq:prescribed_self_distillation_loss} without a teacher.
Following~\citet{lee2026flowmaplanguagemodels}, we sample three times $s<u<t$ by first sampling $d=t-s\sim\mathsf{U}[0,1]$, then sampling $s\sim\mathsf{U}[0,1-d]$, and setting $u=(s+t)/2$.
With learned schedulers, the state at time $s$ is obtained by rolling out the dynamics from time $0$ to $s$ (\Cref{sec:implementation}).
We run this rollout curriculum for five rounds, where the first round rolls out the teacher with 32 steps and each later round rolls out the current flow map with half as many steps, down to 2 steps.
The final round with 2 rollout steps continues for the remaining training steps.
We set the tokenwise time loss weight in \Cref{eq:joint_self_distillation_loss} to $\lambda=1.0$.
For noise downscaling (\Cref{sec:implementation}), we draw the initial noise with $\sigma=0.1$ on GSM8K both in distillation and at inference, and keep $\sigma=1.0$ on Sudoku.

\paragraph{Scheduler hyperparameters.} For the top-$K$ confidence scheduler, we use $(T,A)=(0.2,12)$ on all Sudoku difficulty levels and $(0.5,48)$ on GSM8K, with $k_{\min}=0.1$.
For the entropy-bounded scheduler, we use $(\epsilon,T,A)=(0.01,0.001,12)$ on all Sudoku difficulty levels and $(3,0.3,24)$ on GSM8K, with $k_{\min}=0.1$ throughout. For the prescribed power schedules reported here, we use $\beta=\log 3$, giving exponents in $[1/3,3]$, except for the synchronous schedule, which uses $\beta=0$.
For the initial entropy scheduler, we rank the positions by the entropy of the initial prediction and assign the power schedule with $\beta=\log 3$ according to this rank. 

\paragraph{Evaluation.}
All models are sampled on a uniform grid of the base time $t$ with the given number of steps.
On this grid, the \emph{a}FLM teachers integrate \Cref{eq:state_dependent_dynamics}, \ours{} jumps between adjacent grid points with $\hat F$, and the tokenwise times evolve according to the scheduler.
We evaluate on the full 1,319 GSM8K test questions and 2,000 Sudoku puzzles for each difficulty.
For Sudoku, we report exact solving accuracy, and for GSM8K, we execute the generated code and report the accuracy of its output.
We measure the training cost in \Cref{tab:training_cost} on a single H200 GPU.

\subsection{Baseline details}
\label{app:baseline_details}
For \Cref{tab:main_results_split}, we take the one-step results reported in each paper and otherwise evaluate the checkpoints released by \citet{deschenaux2026language} and \citet{agarwal2026posteriorrefinementfastlanguage}.
For FMLM+, we report the GSM8K result with MDLM initialization.
For DBTM, we report its only one-step GSM8K result, which uses the refinement-in-loop objective.
For AlphaDLM, we report its best one-step result, which the sequence-level objective with its hyperparameter $k=16$ achieves.
Since DBTM and AlphaDLM report no one-step results on Sudoku and release no checkpoints, we mark them with --.

\begin{figure}[t]
    \centering
    \includegraphics[width=\linewidth]{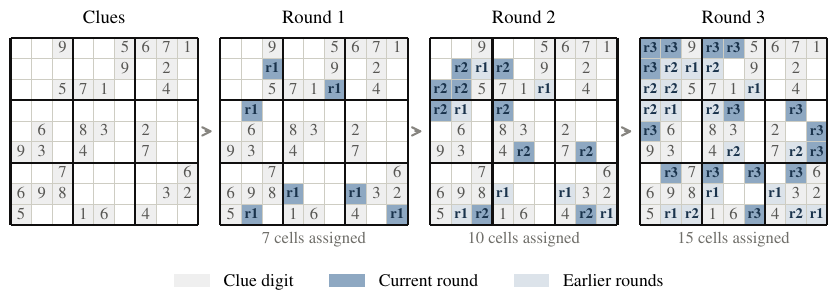}
    \caption{Algorithmic ordering for a hard Sudoku puzzle. Plain digits are the given clues, and labels \texttt{r1}--\texttt{r3} indicate the round in which each empty cell is resolved. Cells resolved in the current round are shaded in darker blue, while previously assigned rounds remain visible. Cells still blank after round 3 are resolved in later rounds (not shown).}
    \label{fig:sudoku_order_example}
\end{figure}
\subsection{Algorithmic ordering for Sudoku}
\label{app:sudoku_order}

Following \citet{shah2024causal}, we start from the given clues ${\bf c}$ and compute the valid digits of each empty cell under the row, column, and box constraints.
At each round, we fill every cell that has only one valid digit or is the only valid location of a digit in its row, column, or box, and we repeat until no cell can be resolved.
\Cref{fig:sudoku_order_example} illustrates the first three rounds.
We denote by $r_{\bf c}(j)$ the round in which cell $j$ is resolved and assign the last round $R_{\bf c}$ to unresolved cells, with $R_{\bf c}=1$ if no cell is resolved.

\subsection{Details of the analysis in \texorpdfstring{\Cref{sec:exp_internalization}}{Section 5.3}}
\label{app:logit_lens}
We apply the logit lens after the attention and MLP sublayers of each transformer block, decoding each hidden state with the final layer norm and output head of the model.
For each token, we take the earliest sublayer from which its decoded prediction matches the final output and remains so.
For the \emph{a}FLM teacher, we run 128-step EB sampling from the same initial noise and regard a token as decided at the first step where its tokenwise time exceeds 0.95.
The four columns of \Cref{fig:exp_internalization} show the first 10\%, 50\%, 85\%, and 100\% of the tokens in this order, and \Cref{tab:fig5_blocks} lists the block of each \ours{} column.

\begin{table}[t]
    \centering
    \caption{Transformer block (out of 12) at which each column of \Cref{fig:exp_internalization} is taken.}
    \label{tab:fig5_blocks}
    \begin{tabular}{@{}lcccc@{}}
        \toprule
        Model & 10\% & 50\% & 85\% & 100\% \\
        \midrule
        \ours{} (EB) & 6 & 8 & 11 & 12 \\
        \ours{} (Sync.) & 2 & 7 & 10 & 12 \\
        \bottomrule
    \end{tabular}
\end{table}